\pdfoutput=1
\documentclass{article}
\usepackage{amsmath}

    \usepackage[final]{neurips_2020}

\usepackage{neurips_2020}
\setcitestyle{numbers,square,comma}

\usepackage[utf8]{inputenc} 
\usepackage[T1]{fontenc}    
\usepackage{hyperref}       
\usepackage{url}            
\usepackage{booktabs}       
\usepackage{amsfonts}       
\usepackage{nicefrac}       
\usepackage{microtype}      
\usepackage{hyperref}
\usepackage[utf8]{inputenc}
\usepackage[english]{babel}
\usepackage[english]{babel}
\usepackage{breakcites}
\usepackage{amsthm}
\usepackage{mathabx}
\usepackage{amssymb}
\usepackage{mathtools}
\usepackage{dsfont}
\usepackage{tikz}
\usetikzlibrary{shapes,decorations,arrows,calc,arrows.meta,fit,positioning}
\tikzset{
    -Latex,auto,node distance =1 cm and 1 cm,semithick,
    state/.style ={ellipse, draw, minimum width = 0.7 cm},
    point/.style = {circle, draw, inner sep=0.04cm,fill,node contents={}},
    bidirected/.style={Latex-Latex,dashed},
    el/.style = {inner sep=2pt, align=left, sloped}
}

\title{Provably Efficient Neural Estimation of Structural Equation Model: An Adversarial Approach}

\author{%
  Luofeng Liao \\
  The University of Chicago\\
  \texttt{luofengl@uchicago.edu} \\
  \And
  You-Lin Chen \\
  The University of Chicago\\
  \texttt{youlinchen@uchicago.edu} \\
  \And 
  Zhuoran Yang \\
  Princeton University\\
  \texttt{zy6@princeton.edu}
  \And
  Bo Dai\\
  Google Research, Brain Team\\
  \texttt{bodai@google.com}
  \And 
  Zhaoran Wang \\
  Northwestern University\\
  \texttt{zhaoranwang@gmail.com}
  \And
  Mladen Kolar \\
  The University of Chicago\\
  \texttt{mkolar@chicagobooth.edu}
}
\usepackage[english]{babel}

\newtheorem{assumption}{Assumption}

\newtheorem{example}{Example}

\newtheorem{theorem}{Theorem}[section]

\newtheorem{lemma}[theorem]{Lemma}
\usepackage{microtype}
\def\##1\#{\begin{align}#1\end{align}}
\def\$#1\${\begin{align*}#1\end{align*}}
\usepackage{enumitem}
\usepackage{algorithmic}
\usepackage[ruled,vlined]{algorithm2e}

\DeclareMathOperator*{\argmin}{arg\,min}

\newcommand{\indep}{\perp \!\!\! \perp}

\def\P{{\mathbb P}}
\def\E{{\mathbb E}}

\def\R{{\mathbb R}}
\def\cO{{\mathcal O}}
\def\cE{{\mathcal E}}
\def\cH{{\mathcal H}}
\def\cF{{\mathcal F}}
\def\cU{{\mathcal U}}

\def\init{{\operatorname{init}}}
\def\indi{{\mathds{1} }}
\def\tp{{^\top}}

\def\fnthe{{\nabla_{\theta} f}}

\def\fh{{\widehat f}}

\def\uh{{\widehat u}}

\def\Fnthe{{\nabla_{\theta} F}}

\def\Fnome{{\nabla_{\omega} F}}

\def\Fhnthe{{\nabla_{\theta} \widehat F}}

\def\Fhnome{{\nabla_{\omega} \widehat F}}

\def\Eall{{\E_{\init, X}}}

\begin{document}

\maketitle

\begin{abstract}
Structural equation models (SEMs) are widely 
used in sciences, ranging from economics to psychology,
to uncover causal relationships underlying a complex system
under consideration and estimate structural parameters of interest. 
We study estimation in a class of generalized SEMs where the object 
of interest is defined as the solution to a linear operator equation.
We formulate the linear operator equation as a min-max game, where both 
players are parameterized by neural networks (NNs), and learn the
parameters of these neural networks using the stochastic gradient descent.
We consider both 2-layer and multi-layer NNs with ReLU activation 
functions and prove global convergence in an overparametrized regime, where
the number of neurons is diverging. The results are established using 
techniques from online learning and local linearization of NNs,
and improve in several aspects the current state-of-the-art. For the first 
time we provide a tractable estimation procedure for SEMs
based on NNs with provable convergence and without the need for sample
splitting.
\end{abstract}

\section{Introduction}

 Structural equation models (SEMs) are widely used in 
 economics \citep{wooldridge2010econometric}, 
 psychology \citep{bollen2005structural}, 
 and causal inference \citep{pearl2009causality}.
 In the most general form \citep{pearl2009causality, peters2017elements},
 an SEM defines a joint distribution over $p$
 observed random variables $\{X_j\}_{j=1}^{p}$ as 
 $X_{j}=f_{j}(X_{\mathrm{pa}_{D}(j)}, \varepsilon_{j})$, $j = 1,\dots, p$,
where $\{ f_j\}$ are unknown functions of interest,
$\{\varepsilon_{j}\}$ are mutually independent noise variables, 
$D$ is the underlying directed acyclic graph (DAG), and 
$\mathrm{pa}_{D}(j)$ denotes the set of parents of $X_j$ in $D$. 
The joint distribution of $\{X_j\}$ is Markov with respect to the graph $D$. 

In most cases, estimation of SEMs are based on the conditional moment restrictions implied by the model.
For example, some observational data can be thought of as coming from 
the equilibrium of a dynamic system.
Examples include dynamic models where 
an agent interacts with the environment, 
such as
in reinforcement learning \cite{dann2014policy}, 
consumption-based asset pricing models \citep{escanciano2015nonparametric}, 
and rational expectation models \citep{hansen1982generalized}.
In these models, 
the equilibrium behavior of the agent is characterized by 
conditional moment equations. 
A second example is instrument variable (IV) regression, where conditional moment equations also play a fundamental role. IV regression is used to
estimate causal effects of input $X$ on output $Y$
in the presence of confounding noise $e$ \citep{newey2003instrumental}.
Finally, in time-series and panel data models,
observed variables exhibit temporal or cross-sectional dependence
that can also be depicted by conditioning \citep{su2013nonparametric}.

For these reasons, we study estimation of structural parameters based on the conditional moment restrictions implied by the model. We propose \textit{the generalized structural equation model}, 
which takes the form of a linear operator equation
\# \label{eq:sem} 
Af = b,
\#
where $A:\mathcal{H} \to \mathcal{E}$ is a conditional expectation operator, which in most settings is only accessible by sampling, $\mathcal{H}$ and $\mathcal{E}$ are separable Hilbert spaces of square 
integrable functions with respect to some random variables, 
$f\in \mathcal{H}$ is the structural function of interest, 
and $b\in \mathcal{E}$ is known or can be estimated. 
Section \ref{sec:shortex} provides a number of important examples from
causal inference and econometrics that fit into the framework \eqref{eq:sem}.

Our contribution is threefold. 
\textbf{First}, we propose a new min-max game formulation for estimating $f$ in \eqref{eq:sem},
where we parameterize both players by neural networks (NN). 
We derive a stochastic gradient descent algorithm to learn the parameters 
of both NNs. In contrast to several recent works that rely on RKHS theory
\citep{dai17a, muandet2019dual, singh2019kernel}, our method enjoys expressiveness 
thanks to the representation power of NNs. Moreover, our algorithm does not need sample splitting, which is a common issue in some recent works \citep{hartford2017deep, lewis2018adversarial}. 
\textbf{Second}, we analyze convergence rates of the proposed algorithm in 
the setting of 2-layer and deep NNs using techniques from online learning 
and neural network linearization. We show the algorithm finds a \textit{globally optimal} 
solution as the number of iterations and the width of NNs go to infinity. In comparison, recent works incorporating NNs into SEM \citep{hartford2017deep, lewis2018adversarial, NIPS2019_8615} lack convergence results.
Furthermore, we derive a consistency result under suitable smoothness assumptions
on the unknown function $f$. 
\textbf{Finally}, we demonstrate that our model enjoys wide application in econometric
and causal inference literature through concrete examples, including
non-parametric instrumental variable (IV) regression, 
supply and demand equilibrium model,
and dynamic panel data model.

\subsection{Examples of generalized SEM} \label{sec:shortex}

We describe three examples of generalized SEM: 
IV regression, simultaneous equations models,
and dynamic panel data model. 
In Appendix \ref{sec:semexamples}, we introduce two more 
examples: proxy variables of unmeasured confounders in causal inference \citep{miao2018identifying} and 
Euler equations in consumption-based asset pricing model \citep{escanciano2015nonparametric}. Other examples
that fit into the generalized SEM framework, but are not detailed in the paper,
include nonlinear rational expectation models \citep{hansen1982generalized}, policy evaluation in 
reinforcement learning, inverse reinforcement learning \citep{ng2000algorithms}, 
optimal control in linearly-solvable MDP \citep{dai17a}, and hitting time of stationary process \citep{dai17a}.

\begin{example} [Instrumental Variable Regression \citep{newey2003instrumental, hartford2017deep, horowitz2011applied}] \label{ex:iv} \normalfont
In many applied problems
endogeneity in regressors arises from omitted variables, 
measurement error, and simultaneity \citep{wooldridge2010econometric}. 
IV regression provides a general solution to the
problem of endogenous explanatory variables. 
Without loss of generality, consider the model of the form
\begin{align}
\label{eq:iv}
Y=g_{0}(X)+\varepsilon, \quad \E[\varepsilon \mid Z]=0,
\end{align}
where $g_0$ is the unknown function of interest, 
$Y$ is the response, 
$X$ is a vector of explanatory variables, 
$Z$ is a vector of instrument variables, 
and $\varepsilon$ is the noise term. 
To see how the model fits our framework, define the operator $A: L^2(X) \to L^2(Z)$, $(Ag)(z) = \E[g(X) \mid Z = z]$. Let $b(z) = \E[Y\mid Z = z] \in L^2(Z)$. The structural equation \eqref{eq:iv} can be written as $Ag = b$.

\end{example}

\begin{example}[Simultaneous Equations Models] \label{ex:simul} \normalfont
Dynamic models of agent’s optimization problems or 
of interactions among agents often exhibit simultaneity. 
Consider a demand and supply model as a prototypical example \citep{matzkin2008identification}. 
Let $Q$ and $P$ denote the quantity sold and price of a product, respectively.
Then
\begin{equation}\label{eq:demandandsupply}
    \begin{aligned}
    &Q =D\left(P, I\right) + U_1\,,\,\, P=S\left(Q, W\right) + U_2,
\\
&\E[U_1\mid I,W ] = 0\,, \,\, \E[U_2 \mid I, W] = 0, 
    \end{aligned}
\end{equation}
where $D$ and $S$ are functions of interest, 
$I$ denotes consumers’ income, 
$W$ denotes producers’ input prices, 
$U_1$ denotes an unobservable demand shock, 
and $U_2$ denotes an unobservable supply shock. Each observation of $\{P, Q, I, W\}$ is a solution to the equation \eqref{eq:demandandsupply}. In Appendix \ref{sec:semexamples} we cast it into the form \eqref{eq:sem}. 
The knowledge of $D$ is essential in predicting the effect of financial policy. For example, let $\tau$ be a percentage tax paid by the purchaser. Then the resulting equilibrium quantity is the solution $\hat Q$ to the equation
$\hat Q = D\big((1+\tau)(S(\hat Q, I) + U_1), W \big) + U_2.$ 
\end{example}

\begin{example}[Dynamic Panel Data Models \citep{su2013nonparametric}] \label{ex:dynpanel} \label{ex:panel}
\normalfont
Exploiting how outcomes vary across units and over time in the dataset is a common approach to identifying causal effects \citep{agrawal2019economics}. 
Panel data are comprised of 
observations of multiple units measured over multiple time periods. 
We consider a dynamic model that includes time-varying regressors 
and allows us to investigate
the long-run relationship between economic factors \citep{su2013nonparametric}:
\begin{equation}  \label{eq:paneldatamodel}
    \begin{aligned}
    & Y_{i t}  =m\left(Y_{i, t-1}, X_{i t}\right)+\alpha_{i}+\varepsilon_{i t},
\\
& \E [\varepsilon_{i t} \mid \underline{Y}_{i, t-1}, \underline{X}_{i t} ] =0, \quad i=1, \ldots, N, \quad t=1, \ldots, T. 
    \end{aligned}
\end{equation}
Here $X_{it} $ is a $p \times 1$ vector of regressors, 
$m$ is the unknown function of interest, 
$\alpha_i$'s are the unobserved individual-specific fixed effects, 
potentially correlated with $X_{it}$, 
and $\varepsilon_{it}$'s are idiosyncratic errors. 
$\underline{X}_{i t} \coloneqq (X_{i t} \tp, \ldots, X_{i 1} \tp )\tp$ 
and 
$\underline{Y}_{i, t-1} \coloneqq \left(Y_{i, t-1}, \ldots, Y_{i 1}\right) \tp$ 
are the history of individual $i$ up to time $t$. After first differencing, we can cast \eqref{eq:paneldatamodel} into equation of the form \eqref{eq:sem} (see Appendix \ref{sec:semexamples}).

\end{example}

\subsection{Related work}
\textbf{Neural networks in structural equation models}. IV regression and generalized method of moments (GMM) \citep{hansen1982large} are two important tools in structural estimation. For example, the work of \citet{blundell2007semi} estimates system of nonparametric demand curves with endogeneity and a sieve-based measure of ill-posedness of the statistical inverse problem is introduced. The work of \citet{chen2012estimation} allows for various convex or/and lower-semicompact penalization on unknown structural functions. Typical nonparametric approaches to to IV regression include kernel density estimators \citep{newey2003instrumental, chen2018optimal} and spline regression \citep{darolles2011nonparametric, carrasco2007linear}. However, traditional nonparametric methods usually suffer from the curse of dimensionality and the lack of guidance on the choice of kernels and splines. 

Existing work on structural estimation using NNs, best to our knowledge, includes Deep IV \citep{hartford2017deep}, Deep GMM \citep{NIPS2019_8615} and Adversarial GMM \citep{lewis2018adversarial}. However, due to the artifacts in saddle-point problem derivation and non-linearity of NNs, these methods suffer from computational cost \citep{hartford2017deep}, the need of sample splitting \citep{lewis2018adversarial, hartford2017deep} or lack of convergence results \citep{NIPS2019_8615, lewis2018adversarial}. The use of NN also appears in works in econometrics. The work of \citet{chen2009land} applies NN to estimate unknown habit function in consumption based asset pricing model. The work of \citet{max2018deep} discusses the use of NN in semi-parametric estimation but not computational issues.

Kernel IV \citep{singh2019kernel} and Dual IV \citep{muandet2019dual} apply reproducing kernel Hilbert space (RKHS) theory to IV regression. 
Dual IV is closely related to the work of \citet{dai17a}, where the authors discuss problems of the form $\min_f \mathbb{E}_{x,y}[\ell(y, \mathbb{E}_{z|x}[f(x,z)])]$ and reformulate it as a min-max problem using duality, interchangeability principle, and dual continuity. In Appendix \ref{app:dualivcomment}, we show that our minimax formulation of IV has a natural connection to GMM compared to Dual IV.

Finally, we notice an excellent concurrent work \citep{mini2020dikkala} which discusses the statistical property of a class of minimax estimator for conditional moment restriction problems. The proposed estimator in that paper is almost identical to ours, and yet we focus on showing convergence of training with NN using neural tangent kernel theory. There is significant distinction from their work.

\textbf{Neural tangent kernel and overparametrized NN}.
Recent work on neural tangent kernel (NTK) \citep{jacot2018neural} shows that in the limit when the number of neurons goes to infinity, the nonlinear NN function can be represented by a linear function
specified by the NTK. Consequently, the optimization problem parametrized by NNs reduces to a convex problem, and can be tackled by tools in classical convex optimization. Examples following this idea include \citep{cai2019neural, wang2019neural, xu2019finite}. In fact, the present paper follows a similar philosophy, by reducing the analysis of neural gradient update to regret analysis of convex online learning, in the presence of bias and noise in the gradient. Finally, the present work is also related to recent advances in overparametrized NNs \citep{allen2019convergence, allen2019learning, gao2019convergence, jacot2018neural,lee2019wide, zhang2016understanding, neyshabur2018towards}. These works point out that NNs exhibit an implicit local linearization which allows us to interpret the former as a linear function when they are
trained using gradient type methods. The present paper is
built on an adaptation of these results.

\subsection{Notations}

We call $(f^*, u^*) \in  \cF \times \cU$ a saddle point of a function 
$\phi: \cF \times \cU \rightarrow \R$ if for all 
$f \in \cF$, $u \in \cU$, $\phi(f^*, u) \leq \phi(f^*, u^*) \leq \phi(f, u^*)$. 
The indicator function $\indi \{\cdot\}$ is defined as
$\indi\{A\} = 1$ if the event $A$ is true; otherwise $\indi \{A\} = 0$. 
Let $[n] = \{1,2,\dots, n \}$. For two sequences $\{ a_n\}, \{b_n\}$, 
the notation $b_n = \cO(a_n)$ represents that there exists a constant $C$ 
such that $b_n \leq C a_n$ for all large $n$. We write $a_n \sim b_n$ if $a_n = O(b_n)$ and $b_n = O(a_n)$. The notation 
$\tilde{\cO}$ ignores logarithmic factors. 
For a matrix $A$, let $\|A \|_F$ be the Frobenius norm. 

For a probability space $(\Omega,\cF,\P)$, 
let $X:\Omega \to \R^p$ be a $p$-dimension random vector. 
The probability distribution of $X$ is characterized by its
joint cumulative distribution function $F$. 
Partition $X$ into $X = [X_1^\top, X_2^\top]\tp$ 
where $X_1\in \R^{p_1}, X_2 \in \R^{p_2}$, and 
let $F_{X_1}, F_{X_2}$ be the marginal distribution functions, 
respectively. Denote by 
$L_F^2(\R^{p_1}, F_{X_1}) = \{ f_1:\R^{p_1} \to \R: \E_{X_1}[ f_1(X_1)^2] < \infty \}$ 
the Hilbert space of real-valued square integrable functions of $X_1$ and 
similarly define $L_F^2(\R^{p_2}, F_{X_2})$. 
For ease of presentation we denote $L_F^2(\R^{p_1}, F_{X_1})$ by $L^2(X_1)$ 
when the context is clear. For $f, g \in L_F^2$, 
the inner product is defined by $\langle f,g\rangle_{L^2(X)} = \E_X [f(X) g(X)] $. For a linear operator $A:\cH \to \cE$ denote by $\mathcal{N}(A) = \{f \in \cH : Af = 0\}$ 
its null space. Denote by $A^*$ the adjoint of a bounded linear
operator $A$. For a subspace $B \subset \cH$ in a Hilbert space $\cH$, 
denote by $B^\bot = \{ a\in \cH : \langle a, b\rangle_\cH = 0,\forall b\in B\}$
the orthogonal complement of $B$ in $\cH$.

\section{Adversarial SEM} \label{sec:adversarialsem}

We formalize our problem setup and introduce 
the Tikhonov regularized method for finding a solution 
for the operator equation in \eqref{eq:sem} in Section \ref{sec:setup}. 
In Section \ref{sec:saddlepointformulation} we derive a saddle-point formulation 
of our problem. 
The players of the resulting min-max game are parametrized by NNs, detailed in Section \ref{sec:nnparame}.

\subsection{Problem setup} \label{sec:setup}
Let $X = [X_1^\top, X_2^\top]\tp$ be a random vector with distribution $F_X$. Let $F_{X_1}, F_{X_2}$ be the marginal distributions of $X_1$ and $X_2$, respectively. We assume there are no common elements in $X_1$ and $X_2$. Furthermore, suppose there is a regular conditional distribution for $X_1$ given $X_2$.
Let $\cH = L^2(X_1)$ and $\cE = L^2(X_2)$. We let $A : \cH \to \cE$ be the conditional expectation operator defined as 
\[(Af)(\cdot) = \E[f(X_1) \mid X_2 = \cdot \,].\]
We want to estimate the solution $f$ to the equation \eqref{eq:sem}, $Af = b$, for some known or estimable $b \in \cE$. In statistical learning literature, 
\eqref{eq:sem} is called \textit{stochastic ill-posed problem} 
when $b$ or both $A$ and $b$ have to be estimated \citep{vapnik1998statistical}. 
In the linear integral equation literature, 
when $A$ is compact, \eqref{eq:sem} belongs to the class of Fredholm equations of type I. 
An inverse problem perspective on conditional moment problems 
is provided in \citep{carrasco2007linear}.


A compact operator\footnote{See Appendix \ref{app:compactofcondexpop} for a discussion of 
when a conditional expectation operator is compact.} with infinite dimensional range cannot have a continuous inverse \citep{carrasco2007linear}, 
which raises concerns about stability of operator equation \eqref{eq:sem}. 
A classical way to overcome the problem of instability is to look for a 
Tikhonov regularized solution, which is uniquely defined \citep{kress1989linear}. 
For all $b \in \mathcal{E}$, $\alpha > 0$, we define 
the Tikhonov regularized problem
\# \label{eq:tikhnovov2}
f^\alpha = \argmin_{f \in \mathcal{H}}  \tfrac{1}{2} \|Af - b\|^2_\mathcal{E} + \tfrac{\alpha}{2} \| f\|^2_{\mathcal{H}}.
\#

\subsection{Saddle-point formulation} \label{sec:saddlepointformulation}

From an optimization perspective, problem \eqref{eq:tikhnovov2} is difficult to solve in that the conditional expectation operator is nested inside the square loss. 

First, it is difficult to estimate the conditional expectation in some cases since we have only limited samples coming from the conditional distribution $p(X_1 \mid X_2)$. In the extreme case, for each value of $X_2=x_2$ we only observe one sample.

Second, when approximating $f$ using some parametrized function class, we encounter the so-called double-sample issue. Let's investigate Example \ref{ex:iv}. In the nonparametric IV problem $\E[g(X)| Z] = \E[Y| Z]$, we want to estimate $g$. Consider the square loss $L(g)=\E_Z 
\big[ \big(\E[g(X)-Y| Z] \big)^2 \big]$. Assume $g$ is approximated by a function parameterized with $\theta$. Taking the gradient w.r.t. $\theta$ and assuming exchange of $\nabla_\theta$ and $\E$, we get $\nabla_\theta L(g)=2\E_Z \big[  \E[g(X)-Y| Z] 
\cdot \E[\nabla_\theta g(X) | Z] \big]$. Assume we observe iid samples of $(X,Y,Z)$. The product of two conditional expectation terms implies that, to obtain an unbiased estimate of the gradient, we will need two samples of $(X,Y,Z)$ with $Z$ taking the same value. This is usually unlikely except for simulated environments.

Our saddle-point formulation circumvents such problems by using the probabilistic property of conditional expectation. The proposed method also offers some new insights into the saddle-point formulation of IV regression \citep{muandet2019dual, NIPS2019_8615}, which shows our derivation is closely related to GMM. This is discussed in Appendix \ref{app:dualivcomment}.

Now we derive a min-max game formulation for \eqref{eq:tikhnovov2}. Assume $b$ is known. Let $R:\cH \to \R_+$ be some suitable norm on $\cH$ that captures smoothness of the function $f$. We consider the constrained form of minimization problem \eqref{eq:tikhnovov2}: $\min_{f\in \cH} \,  \frac12 R(f) \text{ subject to } Af = b$. For some positive number $\alpha > 0$, we define the Lagrangian with penalty on the multiplier $u \in \cE$,
\[ \tilde L(f,u) =  \tfrac12 R(f) + \langle Af - b, u \rangle_{\cE}  - \tfrac{\alpha}{2} \| u \|_{\cE }^2.\]
Without loss of generality, we move the penalty level $\alpha$ to $R(f)$. Finally, using a property of conditional expectation that 
$
\langle Af, u \rangle_\cE = \E_{X_2} \big[ \E[f(X_1) \mid X_2] u(X_2) \big] = \E[f(X_1) u(X_2)]
$
and choosing $R(f) = \| f\|_\cH^2$, we arrive at our saddle-point problem
\# \label{eq:minmax_sem}
\min_{f \in L^2(X_1)}  \max_{u \in L^2(X_2)} \E \big[ \big(f(X_1) - b(X_2) \big)  u(X_2) + \tfrac{\alpha}{2} f(X_1)^2  - \tfrac{1}{2} u(X_2)^2 \big]. 
\#

We remark that as long as $R(f)$ can be estimated from samples, our subsequent algorithm and analysis work with some adaptations. Note that the above derivation is also suitable for equations of the form $(I-K)f = b$, where $K$ is a conditional expectation operator (e.g., Example \ref{ex:ccap} in Appendix \ref{sec:semexamples}). Moreover, 
the function $b$ can be either known or estimable from the data, 
i.e., $b$ can be of the form $b(X_2) = \E[ \tilde{ b}(X_1, X_2) \mid X_2 ]$ 
where $\tilde {b}$ is known. 

\section{Neural Network Parametrization} \label{sec:nnparame}

The recent surge of research on the representation power of NNs \citep{jacot2018neural, allen2019convergence, allen2019learning, lee2019wide, arora2019fine, cai2019neural}  motivates us to use NNs as approximators in \eqref{eq:minmax_sem}. Consider the 2-layer NN with parameters $B$ and $m$ (to be defined in \eqref{eq:twolayernn}). As the width of the NN, $m$, goes to infinity, the class NNs approximates a subset of the reproducing kernel Hilbert space induced by the kernel $K(x, y)=\mathbb{E}_{a \sim \mathcal{N}(0, \frac1d I_d)}\left[\indi \left\{a^{\top} x>0\right\} \indi \left\{a^{\top} y>0\right\} x^{\top} y\right]$. Such a subset is a ball with radius $B$ in the corresponding RKHS norm. This function class is sufficiently rich, if the width $m$ and the radius $B$ are sufficiently large \citep{arora2019fine}. 

However, due to non-linearity of NNs, to devise an algorithm for the NN-parametrized problem \eqref{eq:minmax_sem} with theoretical guarantee is no easy task. In this section, we describe the NN parametrization scheme and the algorithm. As a main contribution of the paper, we then provide formal statements of results on convergence rate and estimation consistency. 

To keep the notation simple, we assume $X_1$ and $X_2$ are of the same dimension $d$. 
We parametrize the function spaces $L^2(X_1)$ and $L^2(X_2)$ in \eqref{eq:minmax_sem} 
by a space of NNs, $\cF_{\text{NN}}$, defined in \eqref{eq:twolayernn} and \eqref{eq:multilayernn}
below. With this parameterization, we write the primal problem in \eqref{eq:tikhnovov2} as 
\# \label{eq:primewithnn}
\min_{f\in \cF_{\text{NN}}} L(f) \coloneqq  \tfrac{1}{2} \|Af - b\|^2_\mathcal{E} + \tfrac{\alpha}{2} \| f\|^2_{\mathcal{H}},
\#
and the min-max problem in \eqref{eq:minmax_sem} becomes
\# \label{eq:minmaxwithnn}
\min_{f\in \cF_{\text{NN}}} \max_{u\in \cF_{\text{NN}}}
 \phi(f,u) \coloneqq \E \big[ \big(f(X_1) - b(X_2) \big)  u(X_2) + \tfrac{\alpha}{2} f(X_1)^2  - \tfrac{1}{2} u(X_2)^2 \big]. 
\#

Problem \eqref{eq:minmaxwithnn} involves simultaneous optimization 
over two NNs. Notice that for each fixed $f$ in the outer minimization, 
the maximum of the inner maximization over $\cE$ is attained at 
$u(\cdot) = \E[f(X_1) \mid X_2 = \cdot ] - b(\cdot) \in \cE$. This can be seen by noting $ \max_{u \in \cE} \,\{ \langle Af-b, u \rangle_\cE - \tfrac12 \| u \|_\cE^2 \} = \tfrac12 \| Af - b\|_\cE^2$.
If for all $f \in \cF_{\text{NN}}$ such maximum is attained in 
$\cF_{\text{NN}}$, then every primal solution $f^* \in \cF_{\text{NN}}$ 
in the saddle point of \eqref{eq:minmaxwithnn}, $(f^*, u^*)$, 
is also an optimal solution to the problem \eqref{eq:primewithnn}. 

Next we introduce the function classes of NNs and the initialization schemes.

\subsection{Neural Network Parametrization}
\textbf{2-layer NNs.} Consider the space of 2-layer NNs with ReLU activations and initialization $\Xi_0 = [W(0), b_1, \dots, b_m]$
\# \label{eq:twolayernn}
\mathcal{F}_{d, B, m}(\Xi_0)=\bigg\{ x\mapsto \frac{1}{\sqrt{m}} \sum_{r=1}^{m} b_{r}  \sigma(W_r \tp x): 
W \in S_{B} \bigg\},
\#
where $\sigma(z) = \indi \left\{z>0\right\}\cdot z$ is the ReLU activation, 
$b_1, \ldots, b_m$ are scalars, and 
\[S_B = \left\{W \in \mathbb{R}^{m d}:\|W-W(0)\|_{2} \leq B\right\}\]
is the $B$-sphere centered at the initial point $W(0)\in \R^{md}$. 
Here we denote succinctly by $W$ the weights of a 2-layer NN stacked into a 
long vector of dimension $md$, and use $W_r\in \R^m$ to access 
the weights connecting to the $r$-th neuron, 
i.e., $W = [W_1 \tp ,\dots, W_m \tp] \tp$. 
Each function in $\cF_{B,m}$ is differentiable with respect to $W$, 
1-Lipschitz, and bounded by $B$. 
We state the following distributional assumption on initialization.

\begin{assumption}[NN initialization, 2-layer, \citep{jacot2018neural}] \label{as:nninit}
Consider the 2-layer NN function space $\mathcal{F}_{d, B, m}(\Xi_0)$ defined in \eqref{eq:twolayernn}. 
All initial weights and parameters, collected as $\Xi_0 = [W(0), b_1, \cdots, b_m]$, are independent,
and generated as $b_{r} \sim \text{Uniform}(\{-1,1\})$
and $W(0) \sim \mathcal{N}\left({0}, \tfrac{1}{d} \mathbf{I}_{d}\right)$. 
During training we fix $\{b_r \}_{r=1}^m$ and update $W$.
\end{assumption}

\textbf{Multi-layer NNs.} The class of $H$-layer NN, $\cF_{d,B,H,m}$ with initialization $\Xi_{H,0} = \big\{ A,\{W^{(h)}(0)\}_{h=1}^{H}, b \big\}$ is defined as 
\begin{multline} \label{eq:multilayernn}
\cF_{d,B,H,m}(\Xi_{H,0}) = \Big\{ 
    x \mapsto b^{\top} x^{(H)} \text{ where } x^{(h)}=\tfrac{1}{\sqrt{m}} \cdot  \sigma(W^{(h)}  x^{(h-1)}), 
    h \in[H], 
    \\
    x^{(0)}=A x:   W \in S_{B,H} 
   \Big\},
\end{multline}
where 
$W=(\operatorname{vec}(W^{(1)})^{\top}, \ldots, \operatorname{vec}(W^{(H)})^{\top})^{\top} \in \mathbb{R}^{H m^{2}}$ 
is the collection of weights $W^{(h)} \in \R^{m^2}$ from all middle layers, 
$x^{(h)}$ is the output from the $h$-th layer, 
$A \in \mathbb{R}^{m \times d}$, 
$b\in \R^m$, 
the function $\sigma$ is applied element-wise, 
and 
\[
S_{B,H}=\big\{W \in \mathbb{R}^{H m^{2}}:\|W^{(h)}-W^{(h)}(0)\|_{\mathrm{F}} \leq B \text { for any } h \in[H]\big\}.
\]
We use the following initialization scheme.

\begin{assumption}[NN initialization, multi-layer, \citep{allen2019convergence, gao2019convergence}] \label{as:nninitmulti}
Consider the space of multi-layer NNs $\cF_{d,B,H,m}(\Xi_{H,0})$ defined in \eqref{eq:multilayernn}. Each entry of $A$ and $\{W^{(h)} (0)\}_{h=1}^{H}$ is independently initialized by $N(0,2)$, and entries of $b$ are independently initialized by $N(0,1)$. Assume $m = \Omega  (d^{3 / 2} B^{-1} H^{-3 / 2} \log ^{3 / 2}(m^{1 / 2}  B^{-1}) )$ and $B= \cO(m^{1 / 2} H^{-6} \log ^{-3} m)$. All initial parameters, $\Xi_{H,0} = \big\{ A,\{W^{(h)}(0)\}_{h=1}^{H}, b \big\}$, are independent. During training we keep $A,b$ fixed and update $W$.
\end{assumption}

We overload notations and denote by $\E_{\init}[\cdot]$ the expectation taken over the random variables $\Xi_0$ or $\Xi_{H,0}$, the randomness of NN initialization.

\subsection{Algorithm} \label{sec:theoryconv}


To describe an implementable algorithm, we rewrite the saddle-point problem \eqref{eq:minmaxwithnn} in terms of NN weights. Denote the weights of NNs $f$ and $u$ in \eqref{eq:minmaxwithnn} by $\theta$ and $\omega$, respectively. Now $\theta$ and $\omega$ play the role of $W$ in \eqref{eq:twolayernn} (resp. \eqref{eq:multilayernn}) since during training only the weights $W$ in \eqref{eq:twolayernn} (resp. \eqref{eq:multilayernn}) are updated. 
For brevity set $f_\theta(\cdot) = f(\theta; \cdot)$ and $u_\omega(\cdot) = u(\omega; \cdot)$. 
With a slight abuse of notation, we use $\phi(\theta,\omega)$ and $\phi(f_\theta,u_\omega)$ (defined in \eqref{eq:minmaxwithnn}) interchangeably.
Note $\phi$ is convex in the NN $f_\theta$ but not in the NN weights $\theta$.
Let $F(\theta, \omega; x_1, x_2) = u_\omega(x_2)  f_\theta(x_1) - u_\omega(x_2) b(x_2) - \frac12 u_\omega^2(x_2) + \tfrac{\alpha}{2} f_\theta^2(x_1)$. The saddle-point problem \eqref{eq:minmaxwithnn} is now rewritten as
\# \label{eq:minmaxwithsb}
\min_{\theta \in S_B} \max_{\omega\in S_B} \phi(\theta, \omega) = \E \big[F(\theta, \omega; X_1, X_2)\big].
\#
\ref{ppoalgo} is the proposed stochastic primal-dual algorithm for solving the game \eqref{eq:minmaxwithnn}. 
Given initial weights $\theta_1$ and $\omega_1$, stepsize $\eta$, 
and i.i.d. samples $\{X_{1,t}, X_{2,t}\}$, for $t = 2,\dots, T-1$,
\begin{equation}\label{ppoalgo}  \tag{Algorithm 1}
\begin{aligned} 
 & \theta_{t+1}=\Pi_{S_B} \big(\theta_{t}-\eta \nabla_{\theta} F(\theta_t, \omega_{t}; X_{1,t}, X_{2,t}) \big ), \\
 &  \omega_{t+1}=\Pi_{S_B} \big (\omega_{t} + \eta \nabla_{\omega} F(\theta_t, \omega_t;  X_{1,t}, X_{2,t}) \big).
 \end{aligned} 
\end{equation}
Here $\Pi_{S_B}$ is the projection operator. The search spaces in \eqref{eq:minmaxwithsb} and the projection operator should be replaced by $S_{B,m}$ and $\Pi_{S_{B,m}}$, respectively, when multi-layer NNs are used. If $b$ takes the form $b(X_2) = \E[ \tilde{ b}(X_1, X_2) \mid X_2 ]$ our algorithm proceeds by replacing $b$ in $F$ with $\tilde{b}$. 

Define by $\mathcal{D}=\sigma \big\{\{X_{1,t}, X_{2,t}\}_{t=1}^T \big\}$ the $\sigma$-algebra generated by the training data. Define the average of NNs as our final output
\#\widebar f_T(\cdot) = \frac{1}{T}\sum_{t=1}^T f(\theta_t; \cdot)\,. \label{eq:nnoutput}\#

\section{Main results}

Due to nonlinearity of NNs, $\phi$ is not convex-concave in $(\theta,\omega)$, which makes the analysis of \ref{ppoalgo} difficult. However, as will be shown in Theorem \ref{thm:globalconvergence}, under certain assumptions \ref{ppoalgo} enjoys \textit{global convergence} as $T$ and $m$ go to infinity. 
For a candidate solution $f$, we consider the suboptimality $E(f)$ as a measure of quality of solution, i.e.,
\# \label{eq:loss}
E( f) = L( f) - L^*,
\#
where $L^* = \min_{f \in \cF_{B,m}} L(f)$ is the minimum value of $L$ over the space of NNs. We define similar quantities when multi-layer NNs are used. Next we describe regularity assumptions on the data distribution.

\begin{assumption}[Bounded support and bounded range] \label{as:bdsupp}
 Assume $ \max\{\|X_1\|_2, \|X_2\|_2\} \leq 1$ almost surely. 
    Assume $b(X_2)$ is bounded almost surely.
\end{assumption}

\begin{assumption}[Regularity of data distribution] \label{as:inputdistn}
    Assume that there exists $c >0$,
    such that
    for any unit vector $v \in \R^{p}$ and any $\zeta >0$, 
    $\P(|v^{\top} X_1| \leq \zeta) \leq c \zeta$, $\P(|v^{\top} X_2| \leq \zeta) \leq c \zeta$. 
 
\end{assumption} 

\begin{assumption}[The conditional expectation operator is closed in $\cF_{\text{NN}}$] \label{as:containtruefunction}
With high probability with respect to NN initialization, for any $f \in \cF_{B,m}$ (or $\cF_{B,H,m}$), $u(\cdot) = \E[f(X_1) \mid X_2 = \cdot \,] - b(\cdot)$ belongs to the class $\cF_{B,m}$ (or $\cF_{B,H,m}$).
\end{assumption}
In Assumption \ref{as:bdsupp}, boundedness of the random variable $b(X_2)$ is satisfied in common applications. Assumption \ref{as:inputdistn} is used when invoking the linearization effects of 2-layer NNs; see Lemma \ref{lm:linearNN}. Assumption \ref{as:containtruefunction} ensures the connection 
between the min-max problem \eqref{eq:minmaxwithnn} and 
the primal problem \eqref{eq:primewithnn}. Assumption \ref{as:containtruefunction} can be removed by incorporating an approximation error term in the error bound. We are ready to state the global convergence results 
for \ref{ppoalgo}. The proof is presented in Appendix \ref{pf:thm:globalconvergence}.

\begin{theorem} [Global convergence of \ref{ppoalgo}]  \label{thm:globalconvergence}
Consider the iterates generated by \ref{ppoalgo} with stepsize $\eta$. 
Let $a = \max \{\alpha,1\}$. Recall $\E_{\init}[\cdot]$ is the expectation w.r.t. NN initialization. For the averaged NN $\widebar f_T$ (defined in \eqref{eq:nnoutput}), its suboptimality $E(\widebar f_T)$ (defined in \eqref{eq:loss}) satisfies the following bounds.

1. (2-layer NNs) Under Assumption \ref{as:nninit}, \ref{as:bdsupp}, \ref{as:inputdistn} and \ref{as:containtruefunction}, with probability over $1 - 2\delta$ with respect to the training data $\mathcal{D}$,
\# \label{eq:errboundtwolayer}
\E_{\init} \big [ E(\widebar f_T) \big ] = \cO\Big( a\eta B + \frac{B}{T\eta} + \frac{a B^{3/2}\log^{1/2}(1/\delta)}{T^{1/2}} + \frac{aB^{5/2}}{m^{1/4}}\Big).
\#
2. (Multi-layer NNs) Under Assumption \ref{as:nninitmulti}, \ref{as:bdsupp} and \ref{as:containtruefunction}, with probability over $1 -c \delta - c\exp\big(\Omega(\log^2 m) \big)$ with respect to the training data $\mathcal{D}$ and NN initialization $\Xi_{H,0}$,
\[ E(\widebar f_T)  = \cO\Big( P_1 \eta a \log m  + \frac{P_2}{T\eta} + \frac{P_3  a \log m \log^{1/2}(1/\delta) }{T^{1/2}} + \frac{P_4a\log ^{3/2} m }{m^{1/6}}\Big),\]
where $P_1 = H^4B^{4/3}$, $P_2 =H^{1/2}B$, $P_3 = H^{5}B^{2}$, $P_4 = H^6B^3$, and $c$ is an absolute constant.
\end{theorem} 

Each of the error rates in Theorem \ref{thm:globalconvergence} consists of two parts: 
the optimization error and the linear approximation error; 
see Section \ref{sec:proof} for a detailed derivation. For the two-layer case, 
if the total training step $T$ is known in advance, 
the optimal stepsize choice is $\eta \sim T^{-1/2}$, and the resulting error rate 
is $\tilde \cO(T^{-1/2} + m^{-1/4})$. The optimization error term $\cO(T^{-1/2})$ is comparable 
to the rate in \citep{nemirovski2009robust} where stochastic mirror descent method
is used in stochastic saddle-point problems. Importantly, the error bound \eqref{eq:errboundtwolayer} 
converges to zero as $T, m \to \infty$. For the multi-layer case, 
optimizing $\eta$ yields the error rate $\tilde \cO(T^{-1/2} + m^{-1/6} )$; 
the linear approximation error has increased due to the highly non-linear nature of multi-layer NNs. 


\subsection{Consistency} \label{sec:theoryconsist}

If we assume smoothness of the solution $f$ to the operator equation $Af = b$ defined in \eqref{eq:sem} and compactness of the operator $A$, we are able to control the rate of regularization bias. 
For a compact operator $A$, 
let $\{ \lambda_j, \phi_j, \psi_j\}_{j=1}^\infty$ be its singular system \citep{kress1989linear}, i.e., $\{\phi_j\}$ and $\{ \psi_j\}$ are orthonormal sequences in $\cH, \cE$, repectively, $\lambda_j \geq 0$, and satisfy
$A \phi_{j}=\lambda_{j} \psi_{j}\,,  A^{*} \psi_{j}=\lambda_{j} \phi_{j}\,, $
where $A^*$ is the adjoint operator of $A$. 
For any $\beta > 0$, define the $\beta$-regularity space \citep{carrasco2007linear}
\# \label{eq:hilberstscale}
\Phi_{\beta}=\Big\{f \in \mathcal{N}( A)^{\perp} \text { such that } \sum_{j=1}^{\infty} \frac{\left\langle f, \phi_{j}\right\rangle^{2}_{\cH}}{\lambda_{j}^{2 \beta}}<\infty\Big\} \subset \cH.
\#
Equipped with the definition of $\beta$-regularity space, 
we are now ready to state the consistency result for 2-layer NN. 
The proof is presented in Appendix \ref{pf:thm:consisttwolayer}.

\begin{assumption}[Zero approximation error]\label{as:zeroapprox}
The primal problems \eqref{eq:tikhnovov2} and \eqref{eq:primewithnn} yield the same solution.
\end{assumption}
\begin{assumption}[Smoothness of the truth]\label{as:reg}
Assume the operator $A$ defined in \eqref{eq:sem} is injective and compact,
and that $f$, the unique solution to \eqref{eq:sem}, lies in the regularity space $\Phi_\beta$ defined in \eqref{eq:hilberstscale} for some $\beta > 0$. 
\end{assumption}

\begin{theorem}[Consistency, 2-layer NN] \label{thm:consisttwolayer}
Consider the iterates generated by \ref{ppoalgo} with stepsize $\eta \sim (aT)^{-1/2}$,
where $a = \max\{\alpha, 1\}$.
Assume \ref{as:nninit}, \ref{as:bdsupp}, \ref{as:inputdistn}, \ref{as:containtruefunction}, \ref{as:zeroapprox} and \ref{as:reg}. Then with probability at least $1 - \delta$ over the sampling process,
    \# \label{eq:esterror}
    \E_\init \big[ \|\widebar f_T - f \|_{L^2(X_1)}^2 \big] = C \Big( \alpha^{\min\{ \beta, 2\}} + \frac{1}{ \alpha \sqrt{a}} \frac{1}{T^{1/2}} +  \frac{a}{\alpha} \big( \frac{1}{T^{1/2}} + \frac{1}{m^{1/4}}\big) \Big) ,
    \#
where $\widebar f_T$ is defined in \eqref{eq:nnoutput}, $f$ in Assumption \ref{as:reg}, and $C$ is a constant independent of $\beta, \alpha, T$ and $m$.
\end{theorem}
If $0< \beta \leq 2$ and $0< \alpha\leq 1$, the optimal choice of $\alpha$ is $\alpha \sim (T^{-1/2}+m^{-1/4})^{1/ ( \beta+1 )}$, assuming $T$ and $m$ are large enough, and the estimation error \eqref{eq:esterror} is of order $\cO \big((T^{-1/2}+m^{-1/4})^{{\beta} / ( \beta+1) } \big)$. To the best of our knowledge, this is the first estimation error rate of structural equation models using NNs. 
We remark \citep{max2018deep} also provides bounds on the estimation error of an NN-based estimator in the setting of semi-parametric inference, but they do not discuss computational issues.

\section{Proof sketch} \label{sec:advsem}

\subsection{Local linearization of NNs}
The key observation is that as the width of NN increases, NN exhibits similar behavior to its linearized version \cite{allen2019learning}. For an NN $f \in \cF_{B,m}$ (or $\cF_{B,H,m}$, with slight notation overload), we denote its linearized version at $W(0)$ by 
\# \label{eq:defNNliear}
\widehat{f}(x,W) = f(x, W(0))+\left\langle\nabla_{W} f \big(x , W(0)\big), W-W(0)\right\rangle. 
\#
The following lemma offers a precise characterization of linearization error for 2-layer NNs; the proof is presented in Section \ref{pf:lm:linearNN} in Appendix \ref{ap:appendix}.  Essentially, it shows that for 2-layer NNs the expected approximation error of the function $f(\,\cdot\,, W)$ by $\widehat f(\,\cdot\,, W)$ decays at the rate $\cO({m}^{-1/2})$, for any $W \in S_B$. In other words, as the width of NN goes to infinity, the NN function behaves like a \textit{linear function}. Similar results on approximation error for multi-layer NNs hold; see Appendix \ref{app:multi}.
\begin{lemma}[Error of local linearization, 2-layer] \label{lm:linearNN}
Consider the 2-layer neural networks in \eqref{eq:twolayernn}. Assume that there exists $c >0$, for any unit vector $v \in \R^d$ and any constant $\zeta >0$, such that $\P_{X}(|v^{\top} X| \leq \zeta) \leq c \zeta$. Under Assumption \ref{as:nninit} we have for all $W \in S_B$ and all $x$,
\$
    & \E_  {\init, X}\big [ | f(X,W) - \widehat f(X,W)|^2 \big] = \cO(B^{3} m^{-1 / 2}), \quad \text{and}
    \\
    & \, \E_{\init, X} \big [ \|\nabla_W f(X,W) - \nabla_W\widehat f(X,W)\|^2 \big] = \cO (B m^{-1 / 2}).
\$
\end{lemma}

\subsection{Convergence analysis} \label{sec:proof}
In this section, we discuss techniques used to bound the minimization error via the analysis of the regret, in the case of 2-layer NNs. The same reasoning applies to the maximizing player $\omega$ and extension to multi-layer NN is obvious. The following lemma relates regret and primal error. The proof is presented in Appendix \ref{pf:thm:boundloss}.
\begin{lemma}[A bound on primal error] \label{lm:boundloss}
Consider a sequence of candidates $  \{ (f_t,u_t ) \}_{t=1}^T$ for the minimax problem \eqref{eq:minmaxwithnn} that satisfy the following regret bounds
\#
\frac1T \sum_{t= 1}^T \phi(f_t, u_t) \leq  \min_{f\in \cF_{\text{NN}} } \frac1T \sum_{t= 1}^T \phi(f, u_t)+ \epsilon_f \label{eq:regrettheta}
\,, \quad
\frac1T \sum_{t= 1}^T \phi(f_t, u_t) \geq \max_{u\in \cF_{\text{NN}} } \frac1T \sum_{t= 1}^T \phi(f_t, u) - \epsilon_u \,.
\#
Denote $\widebar f_T = \frac1T \sum_{t = 1}^T f_t\,$. If Assumption \ref{as:containtruefunction} holds, then $E(\widebar f_T) = L(\widebar f_T) - L^*   \leq \epsilon_f + \epsilon_u\,$.
\end{lemma}
The above lemma suggests we separate our analysis for the two players. For example, to analyze $\epsilon_f$ we can think of the sequence $\{u_t \}$ as fixed and find an upper bound of the quantity $\frac1T \sum_{t= 1}^T \phi(f_t, u_t) - \frac1T \sum_{t= 1}^T \phi(f, u_t)$. We will demonstrate our proof idea via the analysis of $\epsilon_f$; it can easily extend to $\epsilon_u$. 

We focus on the analysis of the minimizer $\theta$ and therefore we denote $\phi_t(\cdot) = \phi(\cdot, \omega_t)$ (defined in \eqref{eq:minmaxwithnn}). 
Also let $\widehat \phi_t(\theta) = \E_{X}[ \widehat u_{\omega_t}  \widehat f_\theta - \widehat u _{\omega_t} b - \frac{1}{2} \widehat u_{\omega_t} ^ 2 + \frac{\alpha}{2} \widehat f_{\theta}^2]$, obtained by replacing $f$ and $u$ in $\phi_t(\cdot)$ with their linearized counterparts defined in \eqref{eq:defNNliear}. The most important property of the linearized surrogate $\widehat \phi _t (\theta)$ is that it is \textit{convex} in $\theta$. To estimate the rate of $\epsilon_f$, we start with the decomposition of regret. For any $\theta \in S_B$, define the regret $\text{Reg}(\theta)=\frac1T \sum_{t= 1}^T \phi_t(\theta_t) - \frac1T \sum_{t= 1}^T \phi_t(\theta)$. Then we have the decomposition
 \refstepcounter{equation} \label{eq:close1}
 \refstepcounter{equation} \label{eq:regret}
 \refstepcounter{equation} \label{eq:close2}
\# \label{eq:threetermsdecomposition}
\text{Reg}(\theta)
 = \underbracket{\frac1T \sum_{t= 1}^T \phi_t(\theta_t) - \frac1T \sum_{t= 1}^T \widehat \phi_t(\theta_t)}_{ \eqref{eq:close1} }
+ \underbracket{\frac1T \sum_{t= 1}^T \widehat \phi_t(\theta_t) - \frac1T \sum_{t= 1}^T \widehat \phi_t(\theta)}_{ \eqref{eq:regret} }
+ \underbracket{\frac1T \sum_{t= 1}^T \widehat \phi_t(\theta) - \frac1T \sum_{t= 1}^T \phi_t(\theta) }_{ \eqref{eq:close2} }.
\#
We bound each term separately. To control the terms \eqref{eq:close1} and \eqref{eq:close2} we use the linearization of NN, which shows that the linearized NN and the original one behave similarly in terms of output and gradient as the width of NN $m$ grows (cf. Lemma \ref{lm:linearNN} and Lemma \ref{lm:multilinearNN}). The term \eqref{eq:regret} is bounded using techniques in convex online learning. The idea is to treat the algorithm designed for solving min-max game associated with $\phi$ as a \textit{biased} primal-dual gradient methods for the one with $\widehat \phi$. We illustrate our techniques in further details in Appendix \ref{app:boundsonregretandapprox}.

\section{Conclusions}
We have derived saddle-point formulation for a class of generalized SEMs and parametrized the players with NNs. We show that the gradient-based primal-dual update enjoys global convergence in the overparametrized regimes ($m \to \infty$), for both 2-layer NNs and multi-layer NNs. Our results shed new light on the theoretical understanding of structural estimation with neural networks.

\section*{Broader Impact}
 In recent years, the impact of machine learning (ML) on economics is already well underway \cite{athey2019generalized, chernozhukov2018double}, and our work serves as a complement to this line of research. On the one hand, machine learning methods such as random forest, support vector machines and neural networks provide great flexibility in modeling, while traditional tools in structural estimation that are well versed in the econometrics community are still primitive, despite recent advances \citep{lewis2018adversarial, hartford2017deep, NIPS2019_8615, max2018deep}. On the other hand, to facilitate ML-base decision making, one must be aware of the distinction between prediction and causal inference. Our method provides an NN-based solution to estimation of generalized SEMs, which encompass a wide range of econometric and causal inference models. However, we remark that in order to apply the method to policy and decision problems, one must pay equal attention to other aspects of the model, such as interpretability, robustness of the estimates, fairness and nondiscrimination, assumptions required for model identification, and the testability of those assumptions. Unthoughtful application of ML methods in an attempt to draw causal conclusions must be avoided for both ML researchers and economists.

\begin{ack}
This work is partially supported by the William S.~Fishman 
Faculty Research Fund at the University of Chicago
Booth School of Business. This work was completed in
part with resources supported by the University of Chicago
Research Computing Center.

We thank Professor Xiaohong Chen at Yale for pointing us to some of the classic works in nonparametric approach to and the use of NN in conditional moment problems, that were omitted in the submission version of this paper.
\end{ack}

\medskip

\bibliographystyle{abbrvnat}
\bibliography{citation.bib}

\clearpage
\begin{appendix} 
\section*{\textsc{Appendices to Provably Efficient Neural Estimation of Structural Equation Model: An Adversarial Approach}}
\section{Examples of generalized structural equation models} \label{sec:semexamples}
In Section \ref{sec:adversarialsem}, we introduce our model in its full generality. Here we specialize it in concrete examples from the causal inference literature and econometrics. 

We remark that the convergence result detailed in Theorem \ref{thm:globalconvergence} applies to all examples while consistency result (Theorem \ref{thm:consisttwolayer}) applies only to Example \ref{ex:iv} because compactness of the conditional expectation operator is required in Theorem \ref{thm:consisttwolayer}.

We add that the paper by \citet[Page 5, Footnote 4]{babii2020iscompleteness} includes a battery economics models that involve conditional moment restrictions, including the measurement error models, dynamic models with unobserved state variables, demand models, neoclassical trade models, models of earnings and consumption dynamics, structural random coefficient models, discrete games, models of two-sided markets, high-dimensional mixed-frequency IV regressions, and functional regression models. We refer readers to the paper for detailed references.

\textbf{Example \ref{ex:iv}, revisited} (\textit{Instrumental Variable Regression}, \citep{newey2003instrumental, hartford2017deep, horowitz2011applied}). \normalfont 
In applied econometrics, endogeneity in regressors usually arises from omitted variables, measurement error, and simultaneity \citep{wooldridge2010econometric}. The method of instrumental variables (IV) provides a general solution to the
problem of an endogenous explanatory variable. Without loss of generality, consider the model of the form
\begin{align}
\tag{\ref{eq:iv} revisited}
Y=g_{0} \left(X\right)+\varepsilon, \quad \E[\varepsilon \mid Z]=0,
\end{align}
where $g_0$ is the unknown function of interest, $Y$ is an observable scalar random variable, $X $ is a vector of explanatory variables, $Z$ is a vector of instrument variables, and $\varepsilon$ is the noise term. For the special case $X = Z$, the estimation of $g_0$ reduces to simple nonparametric regression, since $\E[Y \mid X = x] = g_0(x)$, and can be solved via spline regression or kernel regression \citep{wasserman2006all}. When $X$ is endogenous, which is usually the case in observational data, traditional prediction-based methods fail to estimate $g_0$ consistently. In this case, $g_0(x) \neq \E[Y \mid X = x]$, and prediction and counterfactual prediction become different problems.

To see how the model fits our framework, define the operator $A: L^2(X) \to L^2(Z)$, $Ag = \E[g(X) \mid Z]$. Let $b = \E[Y\mid Z] \in L^2(Z)$. The structural equation \eqref{eq:iv} can be written as $Ag = b$. The minimax problem with penalty level $\alpha$ ($\alpha > 0$) takes the form 
\# \label{eq:ivminmaxwithreg}
\min_{f\in L^2(X)} \max_{u \in L^2(Z)} \E[f(X)u(Z) - Y\cdot u(Z) - \tfrac12u^2(Z) + \tfrac{\alpha}{2} f^2(X)] ,
\#
where the expectation is taken over all random variables.

The IV framework enjoys a long history, especially in economics \cite{greene2003econometric}. It provides a means to answer counterfactual questions like what is the efficacy of a given drug in a given population? What fraction of crimes could have been prevented by a given policy? However, the presence of confounders makes these questions difficult. If $X$ is endogenous, which is usually the case in observational data, then $g_0(x) \neq \E[Y \mid X = x]$, and prediction and counterfactual prediction become different problems. When valid IVs are identified, we have a hope to answer these counterfactual questions.

Counterfactual prediction targets the quantity $\E[Y\mid\text{do}(X = x)]$ defined by the causal graph (see Figure \ref{fig:iv}), where the $\text{do}(\cdot)$ operator indicates that we have intervened to set the value of variable $X$ to $x$ while keeping the distribution of $\varepsilon$ fixed \citep{pearl2009causality}. To facilitate counterfactual prediction, we need to impose stronger conditions on the model \citep{muandet2019dual, hartford2017deep}: (i) relevance: $\P(X \mid Z = z)$ is not constant in $z$; (ii) exclusion: $Y \indep Z \mid X, \varepsilon$; and (iii) unconfounded instrument: $\varepsilon \indep Z$. Figure \ref{fig:iv} encodes such assumptions succinctly.
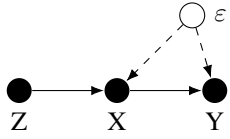
\begin{figure}[ht!]
    \centering
\begin{tikzpicture}
    \node (z) at (-1.3,0) [label=below:Z, circle, fill=black]{};
    \node (x) at (0,0) [label=below:X,circle, fill=black]{};
    \node (y) at (1.3,0) [label=below:Y,circle, fill=black]{};
    \node (u) at (1,1) [label=right:{$\varepsilon$}, circle, draw]{};
    \path (z) edge (x);
    \path (x) edge (y);
    \path (u) [dashed ] edge (x)
              [dashed] edge (y);
\end{tikzpicture}
    \caption{A causal diagram of IV. Three observable variables $X, Y, Z$ (denoted by filled circles) and one unobservable confounding variable $\varepsilon$. There is no direct effect of the instrument $Z$ on the outcome $Y$ except through $X$.}
    \label{fig:iv}
\end{figure}

\textbf{Example \ref{ex:simul}, revisited } (\textit{Simultaneous Equations Models)}. 
Dynamic models of agent’s optimization problems or of interactions among agents often exhibit simultaneity. Demand and supply model is such an example. Let $Q$ and $P$ denote the quantity sold and price of a product. Consider the demand and supply model adapted from \citep{matzkin2008identification}.
\# 
&Q =D\left(P, I\right) + U_1,  \notag
\\ 
&P=S\left(Q, W\right) + U_2, \tag{\ref{eq:demandandsupply} revisited}
\\
&\E[U_1\mid I,W ] = 0, \, \E[U_2 \mid I, W] = 0. \notag
\#
Here $D$ and $S$ are functions of interest, $I$ denotes consumers’ income, $W$ denotes producers’ input prices, $U_1$ denotes an unobservable demand shock, and $U_2$ denotes an unobservable supply shock. Equation \eqref{eq:demandandsupply} is generally the results of equilibrium. Due to simultaneity, there is no hope to recover demand function $D$ by simple nonparametric regression of $Q$ on $P$ and $I$; nor can we recover supply function $S$ by regressing $P$ on $Q$ and $W$. The knowledge of $D$ is essential in predicting the effect of financial policy. For example, let $\tau$ be a percentage tax paid by the purchaser. Then the resulting equilibrium quantity is the solution $\hat Q$ to the equation
\[\hat Q = D\big((1+\tau)(S(\hat Q, I) + U_1), W \big) + U_2.\]

To cast the model \eqref{eq:demandandsupply} to a minimax problem, define the operators 
\$
& A_1 : L^2(P,I)\to L^2(I,W), A_1D = \E[D(P,I)\mid I,W] ,
\\
& A_2 : L^2(Q,W)\to L^2(I,W), A_2S = \E[S(Q,W) \mid I,W].
\$
The resulting minimax problem is
\begin{align*}
    \min_{\substack{D \in L^2(P,I), \\ S\in L^2(Q,W)}} \max_{u_1,u_2 \in L^2(I,W)}
    \left\{
    \substack{
        \E [u_1  (I,W)\cdot(D(P,I) - Q) + u_2(I,W)\cdot (S(Q,W) - P) 
        \\
        \quad \quad \quad - \tfrac{1}{2}u_1(I,W)^2 - \tfrac{1}{2}u_2(I,W)^2 ]
    }
  \right\}.
\end{align*}
Note in this case the operators $A_1$ and $A_2$ are not compact \citep{carrasco2007linear} due to common elements. The min-max derivation remains valid but the stability of the solution is left for future work. 

The causal reading of the simultaneous equations models is an open question since an important assumption often made in causal discovery is that the causal mechanism is acyclic, i.e., that no feedback loops are present in the system \citep{pearl2009causality}. There are efforts in bridging this gap; see, for example, \cite{mooij2011causal}.

\textbf{Example \ref{ex:panel}, revisited} (\textit{Dynamic Panel Data Model}, \citep{su2013nonparametric}). 
\normalfont
Panel data is a common form of econometric data; it contains observations of multiple units measured over multiple time periods. We consider the dynamic model of the following form that includes time-varying regressors, allowing us to investigate
the long-run relationship between economic factors \citep{su2013nonparametric}.
\# 
& Y_{i t}  =m\left(Y_{i, t-1}, X_{i t}\right)+\alpha_{i}+\varepsilon_{i t}, \tag{\ref{eq:paneldatamodel} revisited}
\\
& \E [\varepsilon_{i t} \mid \underline{Y}_{i, t-1}, \underline{X}_{i t} ] =0, \quad i=1, \ldots, N, \quad t=1, \ldots, T, \notag
\#
where $X_{it} $ is a $p \times 1$ vector of regressors, $m$ is the unknown function of interest, $\alpha_i$'s are the unobserved individual-specific fixed effects, potentially correlated with $X_{it}$, and $\varepsilon_{it}$'s are idiosyncratic errors. $\underline{X}_{i t} \coloneqq (X_{i t} \tp, \ldots, X_{i 1} \tp )\tp$ and $\underline{Y}_{i, t-1} \coloneqq \left(Y_{i, t-1}, \ldots, Y_{i 1}\right) \tp$ are the history of individual $i$ up to time $t$. We assume that
$(Y_{it}, X_{it}, \varepsilon_{it})$ are i.i.d. along the individual dimension $i$ but may not be strictly stationary along the time dimension $t$. Clearly, for a large $t$ the conditional set $\{\underline{Y}_{i, t-1}, \underline{X}_{i t} \}$ contains a large number of valid instruments. We do not pursue a search for an efficient choice of IVs in the paper.

To see how it relates to model \eqref{eq:sem}, we consider the first-differenced model
\# 
& \Delta Y_{i t}=m\left(U_{i,t-1}\right)-m\left(U_{i,t-2}\right)+\Delta \varepsilon_{i t}, \label{eq:paneldifferencedvar}
\\
& \E  [ \Delta \varepsilon_{it} \mid U_{i,t-2} ] = 0, \quad i=1, \ldots, N, \quad t=3, \ldots, T, \label{eq:paneldifferenced}
\#
where $\Delta Y_{i t} \coloneqq Y_{i t}-Y_{i, t-1}$, $U_{i, t-2} \coloneqq [Y_{i, t-2}, X_{i, t-1}^{\top}]^{\top}$ and $\Delta \varepsilon_{i t} \coloneqq \varepsilon_{i t}-\varepsilon_{i, t-1}$. The conditional expectation \eqref{eq:paneldifferenced} is obtained by applying law of iterated expectation to \eqref{eq:paneldatamodel} conditional on $U_{i,t-2}$. Model \eqref{eq:paneldifferencedvar} cannot be solved via traditional nonparametric regression because $\Delta \varepsilon_{it}$ is generally correlated with $Y_{i,t-1}$ on the RHS of $\eqref{eq:paneldifferencedvar}$. 

Now we cast the model \eqref{eq:paneldifferenced} into a minimax problem. For ease of exposition we assume strict stationarity on the sequence $\{ U_{it}\}$, which implies that the marginal distribution of $U_{i,t-1}$ and the transition distribution $p(U_{i,t-1} | U_{i,t-2})$ are time-invariant. Now we define a random vector
$(D',E',D,E,F, \varepsilon ) =_d ( Y_{i, t-1}, X_{it}, Y_{i,t-2}, X_{i, t-1}, \Delta Y_{it},  \Delta\varepsilon_{it})$, and the definition is valid due to stationarity. Equation \eqref{eq:paneldifferenced} can be rewritten as
\$
\E [F - m(D',E') + m(D,E) \mid E, D] = 0.
\$
Define the operator $A: L^2(D',E')\to L^2(E,D)$, $Am = \E[m(D',E') \mid E,D]$ and the function $b = \E[F\mid E,D]$. Equation \eqref{eq:paneldifferenced} becomes $(A - I)m = b$, which is a Frehdolm equation of type II. The key difference between type I and type II Fredholm equations lies in stability of the solution. If $I-K: \cH \to \cH$ is injective, then it is surjective, the inverse operator $(I-K)^{-1}$ is continuous and therefore the solution to type II equation is stable \citep{kress1989linear}.

We remark that $1$ is the greatest eigenvalue of $A$ because $(D',E')$ and $(D,E)$ are identically distributed. Therefore we assume the multiplicity of 1 is one in order to identify $m$ up to a constant. The resulting min-max problem is 
\[\min_{m\in L^2(D',E')} \max_{u \in  L^2(E,D)} \E[u(E,D) \cdot \big (F - m(D',E') + m(D,E)  \big) - \tfrac{1}{2}u(E,D)^2].\]
In the absence of the lagged term $Y_{i,t-1}$ on the RHS of \eqref{eq:paneldatamodel}, the model \eqref{eq:paneldatamodel} reduces to the nonparametric panel data model \citep{henderson2008nonparametric},
\$
Y_{i t}=m\left(X_{i t}\right)+\alpha_{i}+\varepsilon_{i t}, \quad i=1, \ldots, n, \quad t=1, \ldots, T.
\$
If the lag term does not appear, we recover the measurement error model studied in \citep{carrasco2007linear}.


\begin{example}[Euler Equation and Utility, \citep{escanciano2015nonparametric}] \label{ex:ccap} \normalfont
In economic models, the behavior of an optimizing agent can be characterized by Euler equations \citep{hansen1982generalized}. Consumption-based capital asset pricing model (CCAPM) is such an example. Here we consider a simplified setting of \citep{escanciano2015nonparametric} where at time $t$ an agent receives  income $W_t$ and purchases or sells certain units of an asset at price $P_t$. For simplicity we assume there is only one asset on the market. Let $U$ be a time-invariant utility function, and $b\in(0,1)$ be the discount factor. $U$ and $b$ are parameters of interests known to the agent but unknown to the researchers. The stream of consumption $\{C_t\}$ is the solution to the optimization problem
\# \label{eq:ccapm}
& \max_{ \{C_t , Q_t \}_{t=0}^\infty} \E \bigg[\sum_{t=0}^\infty \beta^t U(C_t) \bigg]
\\
& \text{s.t.  } C_{t}+ P_{ t} Q_{t}=  P_{ t} Q_{ t-1}+W_{t}, \label{eq:eulerconstraint}
\#
where $Q_t$ is the quantity of the asset owned by the agent at time $t$. RHS of the constraint \eqref{eq:eulerconstraint} is the total value owned by the agent before the exchange at time $t$, while the LHS represents the total value after the exchange. The agent manipulates his consumption, $C_t$, and the quantity of the asset he holds, $Q_t$, to maximize his expected long-run discounted utility.

Define $R_t = P_{t+1} / P_t$. Using the method of Lagrange multiplier, one can obtain the optimality condition of \eqref{eq:ccapm}
\# \label{eq:optimalityconditionccapm}
\E \left[R_{ t+1} \beta \frac{U^{\prime}\left(C_{t+1}\right)}{U^{\prime}\left(C_{t}\right)}-1 \mid I_{t}\right]=0,
\#
where $I_t$ represents the information available at time $t$. A derivation can be found in \citep{escanciano2015nonparametric}.
Let $g = U'$ be the marginal utility function. Conditioning on $C_t$ in \eqref{eq:optimalityconditionccapm}, we obtain 
\# \label{eq:ccapmoperatereq}
\E[\beta R_{t+1} g(C_{t+1}) \mid C_t] = g(C_t).
\#

The goal to estimate the function $g$ given $\{ C_t, R_{t+1}, C_{t+1}\}$. To see how our min-max derivation applies, define the operator $A: L^2(C_{t+1}) \to L^2(C_t)$, $(Ag)(c) = \E[g(C_{t+1}) R_{t+1} \mid C_t = c]$. We assume $A$ is well-defined. Then \eqref{eq:optimalityconditionccapm} can be succinctly written as
\[ \beta Ag = g.\]
We remark that $g$ is identified up to an arbitrary sign and scale normalization; \citep{escanciano2015nonparametric} provides a detailed discussion on identification. Assuming $\beta$ is known, the resulting min-max problem is 
\[
\min_{g\in L^2(C_{t+1})} \max_{u\in L^2(C_t)} \E \big[ \beta g(C_{t+1}) R_{t+1} u(C_t) - g(C_t) u(C_t) - \tfrac{1}{2} u^2(C_t) \big].
\]
One caveat is that $g = 0$ is a trivial solution to \eqref{eq:ccapmoperatereq} and therefore during the training of NNs we should avoid such a solution. The empirical performance of \ref{ppoalgo} in this example is left for future work.

\end{example}

\begin{example}[Proxy Variables of an Unmeasured Confounder, \citep{miao2018identifying}] \label{ex:proxyvar} \normalfont Consider the causal DAG in Figure \ref{fig:proxyvar} in the sense of \citet{pearl2009causality}. Here $X$ and $Y$ denote the treatment and the outcome, respectively. The confounder $U$ is unobserved, while its proxies $Z$ and $W$ are observed. Assume $U, W, Z$ are continuous and in the discussion we assume $X$ and $Y$ are fixed at $(x,y)$. The conditional independence encoded in Figure \ref{fig:proxyvar} is $W \indep(Z, X)\mid U$ and $Z \indep Y \mid (U, X)$. Using the $\text{do}$-operator of \citet{pearl2009causality}, the causal effect of $X$ on $Y$ is 
\[ p ( y | \operatorname{do}(x) )=\int p (y | x, u) p (u) du,\]
where $p(\cdot)$ stands for probability mass functions of a discrete variable or the probability density function for a continuous variable. However, $U$ is unobserved so we cannot directly calculate the causal effect.

The work of \citet{miao2018identifying} provides an identification strategy for the causal effect of $X$ on $Y$ with the help of the confounder proxies $Z$ and $W$. Consider the solution $h(w, x, y)$ to the following integral solution: for all $(x,y)$ and for all $z$, 
\# \label{eq:existenceinproxy} p (y | z, x)=\int_{-\infty}^{+\infty} h(w, x, y) p(w | z, x) d w,\#
which is a Fredholm integral equation of the first kind.

\begin{lemma}[Theorem 1 of \citep{miao2018identifying}]\label{lm:proxyvar} Assume the causal DAG in Figure \ref{fig:proxyvar} and that a solution to \eqref{eq:existenceinproxy} exists. Assume the following completeness condition: $\E [ g(U) | Z,X ]=0$  almost surely if and only if $g(u)=0$ almost surely. Then $p ( y | \operatorname{do}(x) )=\int_{-\infty}^{+\infty} h(w, x, y) p(w) d w$.
\end{lemma}

The result suggests that one can identify the causal effect by first solving for $h$ in \eqref{eq:existenceinproxy} and then applying Lemma \ref{lm:proxyvar}, since $p(y|z,x)$, $p(w|z,x)$ and $p(w)$ can be estimated from the data. To see how \eqref{eq:existenceinproxy} fits into our framework, we note that Equation \eqref{eq:existenceinproxy} implies $\E[\indi\{ Y = y\} \mid Z, X] = \E[h(W,X,y) \mid Z, X]$ for all $y$, and thus similar min-max problem derivation goes through. However, in \citep{miao2018identifying} the identification strategy is limited to the case where $X$ and $Y$ are categorical, and it would be interesting to see how our method performs in the setting of continuous treatment and continuous outcome.

\begin{figure}[ht!]
    \centering
\begin{tikzpicture}
    \node (x) at (-1,0) [label=below:X, circle, fill=black]{};
    \node (y) at (1,0) [label=below:Y,circle, fill=black]{};
    \node (w) at (1.3,1) [label=above:W,circle, fill=black]{};
    \node (u) at (0,1) [label=above:U, circle, draw]{};
    \node (z) at (-1.3,1) [label=above:Z, circle, fill=black]{};

    \path (u) edge (z);
    \path (u) edge (w);
    \path (u) edge (y);
    \path (u) edge (x);
    
    \path (z) edge (x);
    \path (x) edge (y);
    \path (w) edge (y);
\end{tikzpicture}
    \caption{A causal graph of confounder proxies. Adapted from Figure 1(f) of \citep{miao2018identifying}.}
    \label{fig:proxyvar}
\end{figure}
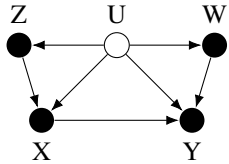

\end{example}

\section{Linear approximation error of multi-layer NNs} \label{app:multi}
Without assumptions on the distribution of data (Assumption \ref{as:inputdistn}), we have slightly worse upper bounds on the error of linearization for multi-layer NNs.
\begin{lemma}[Error of local linearization, multi-layer, \citep{allen2019learning, gao2019convergence}] \label{lm:multilinearNN}
Consider the multi-layer neural networks described in \eqref{eq:multilayernn}. Under Assumption \ref{as:nninitmulti}, with probability at least $1-\exp (-\Omega(\log ^{2} m))$ with respect to the random initialization, for any $W \in S_{B,H}$ and all $x$ such that $\| x \| = 1$,
\begin{enumerate}
    \item $ | \widehat f(x,W) | = \cO(B H^{3 / 2} \log m)$, 
    \item $  \| \nabla_W f(x,W) \|  = \cO(H)$,
    \item $ | f(x,W) - \widehat f(x,W)| = \cO(B^{4 / 3} m^{-1 / 6} H^{3} \log ^{1 / 2} m)$, and
    \item $ \|\nabla_W f(x,W) - \nabla_W\widehat f(x,W)\| = \cO (B^{1 / 3} m^{-1 / 6} H^{5 / 2} \log ^{1 / 2} m)$.
\end{enumerate}
\end{lemma}
\begin{proof}
See Section \ref{pf:lm:multilinearNN} in Appendix \ref{ap:appendix}. 
\end{proof}

\section{Bounds on the terms \eqref{eq:close1}, \eqref{eq:regret} and \eqref{eq:close2}} \label{app:boundsonregretandapprox}

\subsection{Bounds on the terms \eqref{eq:close1}, \eqref{eq:close2}}

First, we establish the closeness between the original function $\phi$ and the one consists of linearized NNs, $\widehat \phi$. The following lemma shows that $\widehat \phi$ is a good surrogate for $\phi$ in the sense that the approximation error is of order $\cO(aB^{5/2}m^{-1/4})$, which vanishes as $m \to \infty$.

Denote $F(\theta, \omega; X_1, X_2) = u_\omega  f_\theta - u_\omega b - \frac12 u_\omega^2 + \tfrac{\alpha}{2} f_\theta^2$. Note $\E_{X} [F(\theta, \omega; X_1, X_2)] = \phi(\theta,\omega)$. Similarly we define $\widehat F(\theta,\omega; X_1,X_2) =\widehat u_{\omega}  \widehat f_\theta - \widehat u _{\omega} b - \frac{1}{2} \widehat u_{\omega}^2 + \tfrac{\alpha}{2}  \widehat f_{\theta} ^2$. 

\begin{lemma}[Closeness between $\widehat \phi$ and $\phi$] \label{lm:closeness}
Let $a = \max \{1,\alpha \} $. For any $\theta, \omega \in S_B$, we have \[\E_{\init} \big [ | \widehat \phi(\theta, \omega) - \phi(\theta, \omega) | \big ] = \cO \big(a B^{5/2}m^{-1/4}\big).\]
\end{lemma}
\begin{proof}
See Section \ref{pf:lm:closeness} in Appendix \ref{ap:appendix}. The proof relies on the decay rates of approximation error, as detailed in Lemma \ref{lm:linearNN}.
\end{proof}
Lemma \ref{lm:closeness} suggests it suffices to set 
\# \label{eq:epsilonfhalf}
\epsilon_f = \cO(aB^{5/2}m^{-1/4}) + \max_\theta \Big(\frac1T \sum_{t= 1}^T \widehat \phi_t(\theta_t) - \frac1T \sum_{t= 1}^T \widehat \phi_t(\theta) \Big).\#
We now turn to bound the term \eqref{eq:regret} using techniques adapted from convex online learning analysis.

\subsection{A bound on the term \eqref{eq:regret}}
We emphasize we apply online learning analysis (Lemma \ref{lm:onlinecvxlearning}) to the regret associated with $\widehat \phi_t$'s but using updates designed for $\phi_t$'s.

\begin{lemma}[Online convex learning with noisy and biased gradient] \label{lm:onlinecvxlearning}
Given a sequence of convex functions on a convex space $\Theta$, $f_1,f_2,\cdots:$ $\Theta \rightarrow \R$, consider the projected gradient descent updates
\# 
\theta_{t+1} &= \Pi_\Theta \big( \theta_{t}-\eta\left(\zeta_{t}+\xi_{t}\right) \big),
\#
where $\mathbb{E}\left[\zeta_{t} | \theta_{t}\right]=\nabla f_{t}\left(\theta_{t}\right)$, $\Pi_\Theta(\theta) \in \operatorname{argmax}_{\theta' \in \Theta} \|\theta - \theta' \|$ is the projection map to $\Theta$. Assume $\sup _{t}\left\|\zeta_{t}+\xi_{t}\right\|<K$ a.s. and $\sup _{\theta} \|\theta\|<M$. Then with probability at least $1-\delta$,
\# \label{eq:041900521}
\frac{1}{T} \sum_{t=1}^{T} f_{t}\left(\theta_{t}\right)-\frac{1}{T} \sum_{t=1}^{T} f_{t}(\theta) \leq \frac{\eta K}{2}+\frac{M}{T \eta}+8 K \sqrt{\frac{M \ln (1 / \delta)}{T}}+\frac{2 \sqrt{2 M}}{T} \sum_{t=1}^{T}\left\|\xi_{t}\right\|
\# for all $\theta \in \Theta$.
\end{lemma}
\begin{proof}
See Section \ref{pf:lm:onlinecvxlearning} in Appendix \ref{ap:appendix}.
\end{proof}
In order to apply Lemma \ref{lm:onlinecvxlearning} to analyze the regret generated by the sequence $\{\widehat \phi_t\}$ with actual updates being $\nabla_{\theta} F_t(\theta_t; X_{1,t}, X_{2,t})$ instead of $\nabla_\theta \widehat \phi_t(\theta_t)$, we need to verify two conditions: (i) bounded update steps, i.e., $\|\nabla_{\theta} F_t(\theta_t; X_{1,t}, X_{2,t})\|$ is bounded for all $t$, and (ii) bounded parameter space. 

To achieve global convergence, we also require that bias in updates, $\| \nabla_{\theta} F_t(\theta_t; X_{1,t}, X_{2,t}) - \nabla_\theta \widehat \phi_t(\theta_t)\|$, which corresponds to the $\|\xi_t\|$ term in \eqref{eq:041900521}, converges to zero as $m \to \infty$. In our analysis we assume $\nabla_\theta \widehat F_t(\theta; X_{1,t}, X_{2,t}) $ is an unbiased estimate of $\nabla \widehat \phi_t(\theta)$.
The following lemma summarizes the results we need to apply Lemma \ref{lm:onlinecvxlearning} and obtain a bound on the term \eqref{eq:regret}.

\begin{lemma}[Bounded gradient and vanishing bias] \label{lm:bdgradsmallbias}
Consider the updates in algorithm \eqref{ppoalgo}. For all $\omega_t$, $\theta$, the following holds.
\begin{enumerate}
    \item $\| \nabla_\theta F_t(\theta; x_1, x_2) \| = \cO(aB)$ for all $x, y$, and 
    \item $\E_{\init, X} \big [ \| \nabla_\theta F(\theta, \omega_t; X_1, X_2) - \nabla_\theta \widehat F(\theta, \omega_t; X_1, X_2)\| \big ] = \cO(a B^{3/2} m^{-1/4})$.
\end{enumerate}
\end{lemma}
\begin{proof}
See Section \ref{pf:lm:bdgradsmallbias} in Appendix \ref{ap:appendix}.
\end{proof}
Equipped with Lemma \ref{lm:closeness} and Lemma \ref{lm:bdgradsmallbias}, we are now ready to obtain a bound on the regret $\epsilon_f$ defined in \eqref{eq:regrettheta}.  Set $M = B$, $K = aB$, $\| \xi_t\| =\cO(aB^{3/2}m^{-1/4})$ in the RHS of \eqref{eq:041900521}, continue \eqref{eq:epsilonfhalf}, and we obtain with probability at least $1-\delta$ with respect to sampling process,
\[\E_\init [\epsilon_f] = \underbracket{\cO\big(aB^{5/2}m^{-1/4}\big)}_{\text{linearization error }\eqref{eq:close1} \text{ and }\eqref{eq:close2}} + 
\underbracket{\cO\Big( \frac{a\eta B}{2} + \frac{B}{T\eta} + \frac{aB^{3/2}\log^{1/2}(1/\delta)}{T^{1/2}} + \frac{aB^4}{m^{1/4}}\Big)}_{\text{optimization error }\eqref{eq:regret}}. \]
It can be shown $\epsilon_u$ is of the same order, thus completing the proof of claim 1 in Theorem \ref{thm:globalconvergence}.

\section{Proof of theorems}\label{ap:appendix}

\textbf{A remark on notations.} Throughout the proof we ignore dependence on $\theta,\omega, X_1$, $X_2$ and the NN initial parameters $\Xi_0$ or $\Xi_{H,0}$ defined in \eqref{eq:twolayernn} and \eqref{eq:multilayernn}, respectively. For readers' convenience, we now restate the dependence of all the functions on their parameters. Recall the NN $f_\theta(X_1) = f(\theta;X_1)$ is an NN with weights $\theta$ and input $X_1$ and similarly for $u_\omega(X_2) = u(\omega;X_2)$. Note $f_\theta$ and $u_\theta$ depend on the initialization implicitly through the range of NN weights (which is centered around the initial weight) and the output layer weights (and the input layer weight, too, in the case of multi-layer NNs). Recall
\[
    \phi = \phi(\theta, \omega) = 
 \phi(f_\theta,u_\omega) \coloneqq \E \big[ \big(f(\theta;X_1) - b(X_2) \big)  u(\omega; X_2) + \tfrac{\alpha}{2} f(\theta; X_1)^2  - \tfrac{1}{2} u(\omega; X_2)^2 \big],
\] and
\[ F = F(\theta, \omega; X_1, X_2)  = \big(f(\theta;X_1) - b(X_2) \big)  u(\omega; X_2) + \tfrac{\alpha}{2} f(\theta; X_1)^2  - \tfrac{1}{2} u(\omega; X_2)^2,\]
and they satisfy $\phi(\theta, \omega) = \E_{X_1, X_2}[ F(\theta, \omega; X_1, X_2) ]$. Note $\phi$ is convex-concave in $(f,u)$ but not in $(\theta, \omega)$. Recall the linearized counterparts of $f$ and $u$, defined in \eqref{eq:defNNliear}, are $\widehat f_{\theta} = \widehat f(\theta(0); X_1) + \langle \nabla_\theta f(\theta(0), X_1), \theta - \theta(0) \rangle$ and similarly for $\widehat u_\omega$. Now we replace NNs $f_\theta$ and $u_\omega$ by their hat-versions in the definition of $\phi$ and $F$ and obtain
$\widehat \phi = \widehat \phi(\theta, \omega, \Xi_0),$
and $\widehat F = \widehat F(\theta, \omega, \Xi_0; X_1, X_2)$. In the proof we only discuss the case where $b = b(X_2)$ is known. The proof goes thorough for the more general case $b(X_2) = \E [\Tilde{b}(X_1, X_2) \mid X_2 ]$ with little modifications.

\subsection{Proof of Lemma \ref{lm:linearNN}} \label{pf:lm:linearNN}
\begin{proof}
The proof follows closely Lemma 5.1 and Lemma 5.2 in \citep{cai2019neural}. Recall that the weights of a 2-layer NN is represented by $W \in \R^{md}$ where $d$ is the input dimension and $m$ is the number of neurons. $W_r \in \R^d$ represents the weights connecting inputs and the $r$-th neuron. $W = [W_1^\top, \dots, W_r^\top]^\top$.

We start with 
\[\left\|\nabla_{W} f(x ; W)\right\|_{2}^2 \leq \frac{1}{m} \sum_{r=1}^{m} \indi \left\{W_r^{\top} x>0\right\}\|x\|_{2}^{2} \leq 1\] for all $W \in S_B$, all $x$. So claim 2 follows. Claim 1 is indeed true because $f(x,W)$ is 1-Lipschitz wrt $W$ and that $\| W - W(0) \|_2 \leq B$ for all $W \in S_B$.
To show claim 3 we first analyze the expression $|f(x,W) - \widehat f(x,W) |$.
\# \label{eq:202004181154}
| f(x & ,W)- \widehat{f}(x, W) |  \notag
\\ 
&=\frac{1}{\sqrt{m}}\left|\sum_{r=1}^{m}\left(\indi \left\{W^{\top}_r x>0\right\}-\indi \left\{W_{r}(0) ^{\top} x>0\right\}\right) \cdot b_{r} W_{r}^{\top} x\right|  \notag
\\
& \leq \frac{1}{\sqrt{m}} \sum_{r=1}^{m}\left|\indi \left\{W_{r}^{\top} x>0\right\}-\indi \left\{W_{r}(0)^{\top} x>0\right\}\right| \cdot\left(\left|W_{r}(0)^{\top} x\right|+\left\|W_{r} -W_{r}(0)\right\|_{2}\right)   \notag
\\
& \leq \frac{1}{\sqrt{m}} \sum_{r=1}^{m} \indi \left\{\left|W_{r}(0)^{\top} x\right| \leq\left\|W_{r}-W_{r}(0)\right\|_{2}\right\} \cdot\left(\left|W_{r}(0)^{\top} x\right|+\left\|W_{r}-W_{r}(0)\right\|_{2}\right)  \notag
\\
& \leq \frac{2}{\sqrt{m}} \sum_{r=1}^{m} \indi \left\{\left|W_{r}(0)^{\top} x\right| \leq\left\|W_{r}-W_{r}(0)\right\|_{2}\right\} \cdot\left\|W_{r}-W_{r}(0)\right\|_{2}.
\#
Here the first inequality follows from $\|x\|_2 = 1$. The second inequality follows from the following reasoning. \[\indi \left\{W_{r}^{\top} x>0\right\} \neq \indi \left\{W_{r}(0)^{\top} x>0\right\}\]
\[\implies \left|W_{r}(0)^{\top} x\right| \leq\left|W_{r}^{\top} x-W_{r}(0)^{\top} x\right| \leq\left\|W_{r}-W_{r}(0)\right\|_{2}.\] The third inequality follows from $\indi \{|x| \leq y\}|x| \leq \indi \{|x| \leq y\} y$ for all $x, y > 0$.

Next we square both sides of \eqref{eq:202004181154}, invoke Cauchy-Schwartz inequality, and the fact that $\|W - W(0)\|_2 \leq B$.
\# \label{eq:202004181147} |f(x,W)-\widehat{f}(x, W)|^{2} \leq \frac{4 B^{2}}{m} \sum_{r=1}^{m} 1\left\{\left|W_{r}(0)^{\top} x\right| \leq\left\|W_{r}-W_{r}(0)\right\|_{2}\right\}.\#

To control the expectation of the RHS of \eqref{eq:202004181147}, we introduce the following lemma.
\begin{lemma} \label{lm:202004181155}
There exists a constant $c_1 > 0$, such that for any random vector $W$ such that $\| W - W(0) \|_2 \leq B$, it holds that 
\[\mathbb{E}_{\init, x}\left[\frac{1}{m} \sum_{r=1}^{m} \indi \left\{\left|W_{r}(0)^{\top} x\right| \leq\left\|W_{r}-W_{r}(0)\right\|_{2}\right\}\right] \leq c_{1} B \cdot m^{-1 / 2} . \]
\end{lemma} 

Taking expectation on both sides of \eqref{eq:202004181147} we get
\[\mathbb{E}_{\text{init}, x }\left[| f(x, W) -\widehat{f}(x , W)|^{2}\right] \leq 4 c_{1} B^{3} \cdot m^{-1 / 2},\]
establishing claim 3. Claim 4 also follows from Lemma \ref{lm:202004181155} as follows.
\$
&  \|\nabla_W f(x,W) - \nabla_W\widehat f(x,W)\|_2^2
\\
& = \frac{1}{m} \sum_{r=1}^{m}\left(1\left\{W_{r}^{\top} x>0\right\}-\indi \left\{W_{r}(0)^{\top} x>0\right\}\right)^{2} \cdot\|x\|_{2}^{2}
\\
& \leq \frac{1}{m} \sum_{r=1}^{m} \indi \left\{\left|W_{r}(0)^{\top} x\right| \leq\left\|W_{r}-W_{r}(0)\right\|_{2}\right\} .
\$
\end{proof}

\textbf{Proof of Lemma \ref{lm:202004181155}}
\begin{proof}
The proof follows Lemma H.1 of \citep{cai2019neural} and is stated for completeness. By the assumption that there exists $c_0 >0$, for any unit vector $v \in \R^d$ and any constant $\zeta >0$, such that $\P_{X}(|v^{\top} X| \leq \zeta) \leq c \zeta$, we have
\#
\mathbb{E}_{\text {init}, x} &\left[\frac{1}{m} \sum_{r=1}^{m} \indi \left\{\left|W_{r}(0)^{\top} x\right| \leq\left\|W_{r}-W_{r}(0)\right\|_{2}\right\}\right]  
\notag
\\
& \leq \mathbb{E}_{\text {init }}\left[\frac{1}{m} \sum_{r=1}^{m} c_{0} \cdot\left\|W_{r}-W_{r}(0)\right\|_{2} /\left\|W_{r}(0)\right\|_{2}\right].   \label{eq:04200703}
\#
Note the expectation in \eqref{eq:04200703} does not involve the data distribution. Next we apply H\"older's inequality.
\$
& \mathbb{E}_{\text {init}, x}\left[\frac{1}{m} \sum_{r=1}^{m} \indi \left\{\left|W_{r}(0)^{\top} x\right| \leq\left\|W_{r}-W_{r}(0)\right\|_{2}\right\}\right]
\\
& \leq  c_{0} / m \cdot \mathbb{E}_{\text {init }}\left[\left(\sum_{r=1}^{m}\left\|W_{r}-W_{r}(0)\right\|_{2}^{2}\right)^{1 / 2} \cdot\left(\sum_{r=1}^{m} \frac{1}{\left\|W_{r}(0)\right\|_{2}^{2}}\right)^{1 / 2}\right]
\\
& \leq c_0 Bm^{-1}  \cdot \mathbb{E}_{\text {init}}\left[\sum_{r=1}^{m} \frac{1}{\left\|W_{r}(0)\right\|_{2}^{2}}\right]^{1 / 2}
\\
& \leq c_0 B m^{-1} \cdot \sqrt{m} \cdot \mathbb{E}_{w \sim N\left(0, I_{d} / d\right)}\left[1 /\|w\|_{2}^{2}\right]^{1 / 2}.
\$
Setting $c_{1}=c_{0} \cdot \mathbb{E}_{w \sim N\left(0, I_{d} / d\right)}\left[1 /\|w\|_{2}^{2}\right]^{1 / 2}$ finishes the proof.
\end{proof}

\subsection{Proof of Lemma \ref{lm:multilinearNN}} \label{pf:lm:multilinearNN}
\begin{proof}
See \citep{allen2019convergence, gao2019convergence} for a detailed proof. Also see Appendix F in \citep{cai2019neural}. In detail, claim 1 follows from equation F.10 of \citep{cai2019neural}. Claim 2 and claim 4 follow from Lemma F.1 of  \citep{cai2019neural}. Claim 3 follows from Lemma F.2 of  \citep{cai2019neural}.
\end{proof}

\subsection{Proof of Lemma \ref{lm:boundloss}} \label{pf:thm:boundloss}
\begin{proof} \label{pf:boundloss}
Recall $\phi(f,u)$ is convex in $f$ and concave in $u$, and that $L(f)$ is convex in $f$. The final output $\widebar f_T$ is the average of the sequence $\{ f_t\}_{t=1} ^ T$ and so is $\widebar u_T$. Recall $\epsilon_f, \epsilon_u$ satisfy
\$
\frac1T \sum_{t= 1}^{T} \phi(f_t, u_t) \leq  \min_{f\in \cF_{NN}} \frac1T \sum_{t= 1}^T \phi(f, u_t)+ \epsilon_f,
\\
\frac1T \sum_{t= 1}^{T} \phi(f_t, u_t) \geq \max_{u\in \cF_{NN}} \frac1T \sum_{t= 1}^T \phi(f_t, u) - \epsilon_u.
\$
Note both $f$ and $u$ range over the space of NNs. We start with the equivalent expression for $L$ defined in \eqref{eq:primewithnn}. By Assumption \ref{as:containtruefunction}, for all $f\in \cF_{NN}$, $L(f) = \max_{u\in \cF_{NN}} \phi(f,u)$ with $\phi$ defined in \eqref{eq:minmaxwithnn}. We have
\$
 & L(\widebar f_T) - L^* 
\\
& = \max_{u\in \cF_{NN}} \phi(\widebar f_T, u) - \min_{f\in \cF_{NN}} \max_{u\in \cF_{NN}} \phi(f,u)
\\
&\leq \max_{u\in \cF_{NN}} \phi(\widebar f_T, u) - \min_{f\in \cF_{NN}} \phi(f,\widebar u_T)
\\
&\leq \max_{u\in \cF_{NN}} \frac1T \sum_{t=1}^{T} \phi( f_t, u) - \min_{f\in \cF_{NN}} \frac1T \sum_{t=1}^{T} \phi( f, u_t)
\\
&= \Bigg[ \Big(  \max_{u\in \cF_{NN}} \frac1T \sum_{t=1}^{T} \phi( f_t, u) \Big) - \frac1T \sum_{t=1}^{T} \phi( f_t, u_t)  \Bigg] 
\\
& \quad +  \Bigg[ \Big( \frac1T \sum_{t=1}^{T} \phi( f_t, u_t)  \Big) - \min_{f\in \cF_{NN}} \frac1T \sum_{t=1}^{T} \phi( f, u_t) \Bigg]
\\
& \leq \epsilon_f + \epsilon_u.
\$

In fact, we easily have $\frac1T \sum_{t=1}^T L(f_i) - L^* \leq \epsilon_f + \epsilon_u$.
\end{proof}
\subsection{Proof of Lemma \ref{lm:closeness}} \label{pf:lm:closeness}
\begin{proof}

Recall $X = [X_1^\top,X_2^\top]^\top$, $\phi(\theta,\omega) = \E_{X}[F(\theta,\omega;X_1,X_2)] = \E_{XY}[u f - u b - (1/2) u ^2 + (\alpha / 2) f^2]$. 

Denote $\widehat F(\theta,\omega) = \uh  \fh - \uh b - (1/2) \uh ^2 + (\alpha / 2) \hat f ^2 $, where the hat-version are the linearized NN. We start by noting
\$
& \E_{\init} \big[ | \widehat \phi(\theta, \omega) - \phi(\theta, \omega) | \big]
 \\
& = \Eall \big[ |\hat F - F\big| ]
\\
& = \Eall \big[ | (\uh  \fh - \uh b - \tfrac12 \uh ^2 + \tfrac{\alpha}{2} \fh^2 ) - (u  f - u b - \tfrac12 u  ^2 + \tfrac{\alpha}{2} f^2 )) |\big]
\\
& \leq \E_{\init, XY} \big[ |\uh  \fh - u  f| \big] + \Eall \big [| (\uh -u)  b|\big]+ (1/2) \Eall \big [| \uh ^2 - u ^2 |\big] +  (\alpha / 2) \Eall \big [ | \widehat f^2  - f^2|\big]. 
\$
Now bound the terms 
\#
&\Eall \big[|\uh  \fh - u  f|\big] , \label{eq:04160941}
\\
&\Eall \big[ |(\uh - u )  b| \big] ,\label{eq:04160942}
\\
&\Eall \big[ |\uh ^2  - u ^2|  \big] ,\label{eq:04160943}
\\
&\Eall \big[ |\widehat f ^2  - f ^2|  \big] .\label{eq:04160943}
\#
For the term \eqref{eq:04160941}, we have
\$
& \E_{\init, X} \big[ |\uh  \fh - u  f| \big]
\\
&\leq \E_{\init, X}\big[ |\uh  (\fh - f)| \big]+ \Eall \big[ (\uh - u)  f \big]
\\
& \leq \sqrt{\E_{\init, X}\big[  \uh ^2\big] \Eall \big[ | \fh - f |^2 \big]} + \sqrt{\E_{\init, X} \big[ f ^2\big]  \Eall \big[| \uh - u |^2 \big] } \tag{Cauchy-Schwarz inequality}
\\
& = \sqrt{\cO (B^2 \cdot B^3 m^{-1/2}) } + \sqrt{\cO(B^3m^{-1/2}) \cdot \cO(B^2)} \tag{Lemma \ref{lm:linearNN}}
\\
& = \cO (B^{5/2}m^{-1/4}).
\$

We can apply similar techniques and obtain the following bounds on \eqref{eq:04160942} and \eqref{eq:04160943}.
\$
& \Eall \big[| (\uh -u)  b | \big] = \cO(B^{3/2}m^{-1/2}),
\\
& \Eall \big[ | \uh  ^2 - u  ^2 | \big] = \cO(B^{5/2}m^{-1/4}).
\$
Putting all pieces together we get
\$
\E_{\init} \big[ | \widehat \phi(\theta, \omega) - \phi(\theta, \omega) | \big] = \cO((1+\alpha)B^{5/2}m^{-1/4}).
\$
\end{proof}

\subsection{Proof of Lemma \ref{lm:onlinecvxlearning}} \label{pf:lm:onlinecvxlearning}
\begin{proof}
We need the following lemma that controls regret in the context of online learning with exact gradient, and then we extend it to our noisy and biased gradient scenario.

\begin{lemma} [Regret analysis in online learning,  \cite{srebro2011ontheuniversality}] \label{lm:04190306}
Let $f_1, f_1, \dots:\Theta \rightarrow \R$ be convex functions, where $\Theta$ is convex. Consider the mirror descent updates, 
\$
\zeta_{t+1} &=\nabla h^{*}\left(\nabla h\left(\theta_{t}\right)-\eta \nabla f_{t}\left(\theta_{t}\right)\right) ,
\\ \theta_{t+1} &=\arg \min _{\theta \in \Theta} D_{h}\left(\theta, \zeta_{t+1}\right) ,
\$
where $h$ is $1$-strongly convex with respect to the norm $\| \cdot \|$, $D_{h}(x, y)=h(x)-h(y)-\nabla h(y)^{\top}(x-y)$ is the Bregman divergence, $h^*$ is the convex conjugate of $h$, and $\|\cdot \|_*$ is the dual norm of $\| \cdot \|$. Suppose that $\sup _{t}\left\|\nabla f_{t}\left(\theta_{t}\right)\right\|_{*}<K$ and $\sup _{\theta} h(\theta)<M$. Then for all $\theta \in \Theta$, \[\frac{1}{T} \sum_{t=1}^{T} f_{t}\left(\theta_{t}\right)-\frac{1}{T} \sum_{t=1}^{T} f_{t}(\theta) \leq \frac{\eta K}{2}+\frac{M}{T \eta}.\]
\end{lemma}

We refer readers to \cite{srebro2011ontheuniversality} for a proof of Lemma \ref{lm:04190306}. Now we take $h(x) = \tfrac{1}{2}\| x \|$, and $\|x\|$ is the Euclidean norm. 

Note that in our case the actual update is $\zeta_{t}+\xi_{t}$, where $\zeta_t$ is an unbiased estimate of the gradient $\nabla f_t(\theta_t)$, and $\xi_t$ is a noise term. We construct linear surrogate functions $\widehat{f}_{t}(\theta)=f_{t}\left(\theta_{t}\right)+\left(\zeta_{t}+\xi_{t}\right)^{\top}\left(\theta-\theta_{t}\right)$ and notice that $\zeta_{t}+\xi_{t}$ is indeed the gradient of the surrogate at $\theta_t$, i.e., $\nabla \widehat{f}_t(\theta_t) = \zeta_{t}+\xi_{t}$. Now we apply Lemma $\ref{lm:04190306}$ to the sequence $\{ \widehat{ f}_t\}$ and obtain
\[\frac{1}{T} \sum_{t=1}^{T} \widehat{f}_{t}\left(\theta_{t}\right)-\frac{1}{T} \sum_{t=1}^{T} \widehat{f}_{t}(\theta) \leq \frac{\eta B}{2}+\frac{M}{T \eta},\]
which implies
\$
 \frac{1}{T} &\sum_{t=1}^{T} f_{t}\left(\theta_{t}\right)-\frac{1}{T} \sum_{t=1}^{T} f_{t}(\theta) 
\\
& \leq \frac{\eta B}{2}+\frac{M}{T \eta}+\frac{1}{T} \sum_{t=1}^{T} \widehat{f}_{t}(\theta)-\frac{1}{T} \sum_{t=1}^{T} f_{t}(\theta)
\\
& \leq \frac{\eta B}{2}+\frac{M}{T \eta}+\frac{1}{T} \sum_{t=1}^{T}\left(\zeta_{t}-\nabla f_{t}\left(\theta_{t}\right)\right)^{\top}\left(\theta-\theta_{t}\right)+\frac{1}{T} \sum_{t=1}^{T} \xi_{t}^{\top}\left(\theta-\theta_{t}\right).
\$
Now we bound the term $\sum_{t=1}^{T}\left(\zeta_{t}-\nabla f_{t}\left(\theta_{t}\right)\right)^{\top}\left(\theta-\theta_{t}\right)$. We note the boundedness of the quantities
\$
\left(\zeta_{t}-\nabla f_{t}\left(\theta_{t}\right)\right)^{\top}\left(\theta-\theta_{t}\right) & \leq\left\|\zeta_{t}-\nabla f_{t}\left(\theta_{t}\right)\right\| 2 \sqrt{2 M} \leq 4 B \sqrt{2 M} .
\$
To control the sum of bounded random variables, we invoke Hoeffding-Azuma inequality, and obtain that for $0 < \delta <1$,
\[\mathbb{P}\left\{\frac{1}{T} \sum_{t=1}^{T}\big(\zeta_{t}-\nabla f_{t}\left(\theta_{t}\right)\big)^{\top}\left(\theta-\theta_{t}\right) \geq 8 B \sqrt{\frac{M \log (1 / \delta)}{T}}\right\} \leq \delta.\]
Finally we have $\xi_{t}^{\top}\left(\theta-\theta_{t}\right) \leq\left\|\xi_{t}\right\| 2 \sqrt{2 M}$. Putting all the pieces together completes the proof.
\end{proof}

\subsection{Proof of Lemma \ref{lm:bdgradsmallbias}} \label{pf:lm:bdgradsmallbias}

\begin{proof}
The gradients of $F$ with respect to $\omega, \theta$ are
\$
\Fnthe & = (u_\omega + \alpha f_\theta) \fnthe_\theta ,
\\
\Fnome & =(f_\theta - b - u_\omega) \nabla_\omega u_{\omega} .
\$
First we show for all $x_1,x_2$, $\omega$ and $\theta$, we have that $\Fnthe$ is bounded. It is easy to see by Lemma \ref{lm:linearNN}
\$
\| \Fnthe \|_2 = \cO((1+\alpha)B).
\$
Next we show that for all $\theta, \omega$, $\E_{\init, X} [ \| \nabla_\theta F_t  - \nabla_\theta \widehat F\|]$ goes to zero as $m\to \infty$.

\$
& \E_{\init, X} \big[ \| \nabla_\theta F - \nabla_\theta \widehat F\|\big]
\\
& \leq  \sqrt{\Eall \big [ \| \nabla_\theta f \|^2 \big]\Eall \big[(u - \hat u )^2 \big] }
\\ 
& \quad \quad + \sqrt{\Eall \big[ \|\nabla_\theta \hat f - \nabla_\theta  f \|^2 \big] \Eall \big[\hat u^2 \big]} 
\\
& \quad \quad+  \alpha\sqrt{\Eall \big [ \| \nabla_\theta f \|^2 \big]\Eall \big[(f - \hat f )^2 \big] }
\\ 
&\quad \quad + \alpha \sqrt{\Eall \big[ \|\nabla_\theta \hat f - \nabla_\theta  f \|^2 \big] \Eall \big[\hat f^2 \big]} 
\\
& = \cO((1 + \alpha)B^{3/2} m ^{-1/4})
\$
\end{proof}

\subsection{Proof of Theorem \ref{thm:globalconvergence}} \label{pf:thm:globalconvergence}

\textbf{Remark.} In fact, the two bounds in Theorem \ref{thm:globalconvergence} are also valid bounds on $\E_{\init}  [  \frac{1}{T} \sum_{t=1}^T E(f_t)  ] $ and $ \frac{1}{T} \sum_{t=1}^T E(f_t)$, respectively. For example, in the 2-layer NN case, it also holds that
\#  \label{eq:05250252}
\E_{\init} \bigg [  \frac{1}{T} \sum_{t=1}^T E(f_t) \bigg ] = \cO\Big( a\eta B + \frac{B}{T\eta} + \frac{a B^{3/2}\log^{1/2}(1/\delta)}{T^{1/2}} + \frac{aB^{5/2}}{m^{1/4}}\Big). \#

During training we obtain a sequence of NN weights $\theta_1, \theta_2, \cdots, \theta_T$ and the corresponding NNs $f_1,f_2,\dots, f_T$. 
The difference lies in that in \eqref{eq:05250252} we bound the \textit{average of the suboptimality} of the NNs $f_1,f_2,\dots, f_T$ rather than the \textit{suboptimality of the averaged NN} $\bar f_T = \frac{1}{T} \sum_t f_t$, as is done in Theorem \ref{thm:globalconvergence}. The bound \eqref{eq:05250252} implies that to choose the output NN it suffices to just pick one NN from the sequence of NNs $f_1,f_2,\dots, f_T$ uniformly.

\textbf{Proof of Theorem \ref{thm:globalconvergence}, two-layer NN} 
\begin{proof}
Based on the analysis in Appendix \ref{app:boundsonregretandapprox}, all we need to do is to estimate the rate of the following quantities
\begin{enumerate}
    \item $\E_{\init} \big[ | \widehat \phi(\theta, \omega) - \phi(\theta, \omega) | \big]=  \cO((1+\alpha)B^{5/2}m^{-1/4})$,
    \item $\sup \|\theta\| = \cO(B)$, $\| \Fnthe \| = \cO((1+\alpha)B)$,
    \item $\sup \|\omega\| = \cO(B)$, $\| \Fnome \| = \cO(B)$,
    \item $\E_{\init, X} \big[ \| \nabla_\theta F - \nabla_\theta \widehat F\|\big] = \cO((1+\alpha)B^{3/2}m^{-1/4})$, and 
    \item $\Eall [ \| \Fnome - \Fhnome\| ] = \cO(B^{3/2}m^{-1/4})$.
\end{enumerate}

The missing pieces are 
\begin{itemize}
    \item $\|\Fnome\|$ is bounded, and
    \item $\Eall \big [ \| \Fnome - \Fhnome\| \big] = \cO(B^{3/2}m^{-1/4})$ .
\end{itemize}
First we bound the term $\| \Fnome\|$. It is easy to see
\$
\|\Fnome\| = \cO(B).
\$
Then we show $\Eall [ \| \Fnome - \Fhnome\| ] = \cO( B^{3/2}m^{-1/4})$
\$
& \Eall \big[ \| \Fnome - \Fhnome\| \big]
\\
& \leq \sqrt{ \Eall[ | f - b - u|^2 ] \Eall [ \| \nabla_\omega \hat u - \nabla_\omega u  \|^2 ] } 
\\
& \quad \quad +\sqrt {\Eall [  | (f - \hat f ) + (u - \hat u )|^2 ] \Eall [ \| \nabla_\omega \hat u\|^2 ] }  \tag{Cauchy-Schwarz inequality}
\\
& = \cO( B^{3/2}m^{-1/4})
\$
\end{proof}

\textbf{Proof of Theorem \ref{thm:globalconvergence}, multi-layer NN} 

\begin{proof}We mimic the same proof technique as the two-layer case. We need to verify with probability at least $1 - \exp(\Omega(\log^2 m))$ over the NN initialization,
\begin{enumerate}
    \item $| \widehat \phi(\theta, \omega) - \phi(\theta, \omega) | = \cO((1+\alpha)B^{8/3}H^{6}  m^{-1/6} \log ^{3/2}m)$, for all $\theta, \omega \in S_{B,H}$,
    \item $\sup \| \theta \|_2 = H^{1/2}B, \| \Fnthe \| = \cO((1+\alpha )B^{4/3}H^4 \log m)$ for all $\theta, \omega \in S_{B,H}$ and $x_1, x_2$,
    \item  $\sup \| \omega \|_2 = H^{1/2}B, \| \Fnome \| = \cO(B^{4/3}H^4 \log m)$, for all $\theta, \omega \in S_{B,H}$ and $x_1, x_2$,
    \item $\E_{X} [ \| \Fnome - \Fhnome\| ] = \cO(B^{4/3}H^4 m^{-1/6} \log ^{3/2}m)$, for all $\theta, \omega \in S_{B,H}$, and
    \item $\E_{X} [ \| \Fnthe - \Fhnthe\| ] = \cO((1+\alpha)B^{4/3}H^4 m^{-1/6} \log ^{3/2}m)$ for all $\theta, \omega \in S_{B,H}$.
\end{enumerate}
To show claim 1, we need to find high probability bounds of the terms
\#
& |\uh  \fh - u  f|  , \label{eq:04211103}
\\
&  |(\uh - u )  b|   ,\label{eq:04211104}
\\
&  |\uh  ^2 - u  ^2|   \label{eq:04211105}
\\
&   |\widehat f  ^2 - f  ^2| 
\#
For the term \eqref{eq:04211103},
\#
&   |\uh  \fh - u  f|   \notag
\\
&\leq  |\uh  (\fh - f)|  +  | (\uh - u)  f  | \notag
\\
& \leq \sqrt{ \| \uh \|^2  \| \fh - f \|^2 } + \sqrt{\| f \|^2 \| \uh - u \|^2} \tag{Cauchy-Schwarz inequality}
\\
& = \sqrt{\cO (B^2H^3\log^2 m \cdot B^{8/3}H^6 m^{-1/3}  \log m) }   \notag
\\
& \quad \quad + \sqrt{\cO(B^{8/3}H^6 m^{-1/3} \log m) \cdot \cO(B^2H^3)} \label{eq:04210121}
\\
& = \cO (B^{7/3}H^{9/2} m^{-1/6}  \log ^{3/2}m), \notag
\#
where equality \eqref{eq:04210121} is valid with probability at least $1 - \exp(\Omega(\log^2 m))$. Similarly we have the following high probability bounds.
\$
 & |(\uh - u )  b| = \cO(B^{4/3}m^{-1/6}H^3 \log^{1/2}m),
\\
 & |\uh  ^2 - u  ^2| = \cO(B^{8/3} m ^{-1/6} H^6 \log^{3/2}m).
\$
Putting all the pieces together completes the proof of claim 1.

For claim 2, $ \| W - W(0) \|_2 \leq \sqrt{H}B$ implies $\sup \| \theta \|_2 \leq  H^{1/2}B$. For $\| \Fnthe\|$,
\$
& \| \Fnthe \|_2 
\\
& = \|(u + \alpha f) \fnthe \|_2
\\
& = \cO((1+\alpha )B^{4/3}H^4 \log m).
\$
This completes proof of claim 2. Claim 3 follows similarly. For claim 4,
\$
&\| \Fnome - \Fhnome\| 
\\
& =  \|   (f - b - u ) \nabla_{\omega} u -  (\widehat f - b - \widehat u ) \nabla_{\omega}\widehat u  \|
\\
& \leq  \sqrt{ | \widehat f - b - \widehat u|^2   \| \nabla_{\omega} \widehat u -\nabla_{\omega} u \|^2  } 
\\
& \quad \quad + \sqrt {  | (f - \widehat f ) + (u - \widehat u )|^2   \| \nabla_{\omega} u\|^2  }  
\\
& = \cO(B^{4/3}H^4 m^{-1/6}\log^{3/2}m)
\$
where the last equality holds with high probability. 
Recall the decomposition \eqref{eq:threetermsdecomposition},
\# 
& \frac1T \sum_{t= 1}^T \phi_t(\theta_t) - \frac1T \sum_{t= 1}^T \phi_t(\theta) \ \tag{\ref{eq:threetermsdecomposition}, revisited}
\\
& = \underbracket{\frac1T \sum_{t= 1}^T \phi_t(\theta_t) - \frac1T \sum_{t= 1}^T \widehat \phi_t(\theta_t)}_{ \eqref{eq:close1} }
+ \underbracket{\frac1T \sum_{t= 1}^T \widehat \phi_t(\theta_t) - \frac1T \sum_{t= 1}^T \widehat \phi_t(\theta)}_{ \eqref{eq:regret} }
+ \underbracket{\frac1T \sum_{t= 1}^T \widehat \phi_t(\theta) - \frac1T \sum_{t= 1}^T \phi_t(\theta) }_{ \eqref{eq:close2} }. \notag
\#
Finally, we put together the pieces. Define the events
\$
E_1 = 
\Bigg\{ &
\underbracket{\frac1T \sum_{t= 1}^T \phi_t(\theta_t) - \frac1T \sum_{t= 1}^T \widehat \phi_t(\theta_t)}_{ \eqref{eq:close1} } +  \underbracket{\frac1T \sum_{t= 1}^T \widehat \phi_t(\theta) - \frac1T \sum_{t= 1}^T \phi_t(\theta) }_{ \eqref{eq:close2} }
\\
& = \cO((1+\alpha)B^{8/3}H^{6}  m^{-1/6} \log ^{3/2}m)
\Bigg\}
\$ and
\$
E_2 = \Bigg\{
& \underbracket{\frac1T \sum_{t= 1}^T \widehat \phi_t(\theta_t) - \frac1T \sum_{t= 1}^T \widehat \phi_t(\theta)}_{ \eqref{eq:regret} }
\\
& = \cO( {P_1 \eta a \log m } + \frac{P_2}{T\eta} + \frac{a P_3 \log m \log^{1/2}(1/\delta) }{T^{1/2}} + \frac{P_3a\log ^{3/2} m }{m^{1/6}}) \Bigg\}
\$
where $P_1 = H^4B^{4/3}$, $P_2 =H^{1/2}B$ and $P_3 = H^{5}B^{2}$, defined in Theorem \ref{thm:globalconvergence}.
By claim 1 we have $\P(E_1) \geq 1 - \exp(\Omega(\log^2 m))$. By claim 2, claim 5 and Lemma \ref{lm:onlinecvxlearning} we have $\P(E_2) \geq 1-\delta - \exp(\Omega(\log^2 m))$. Then $\P(E_1 \cap E_2)) \geq 1 - c \delta -c \exp(\Omega(\log^2 m))$ for some absolute constant $c$. The same analysis applies for $\omega$ and therefore we complete the proof.
\end{proof}

\subsection{Proof of Theorem \ref{thm:consisttwolayer}} \label{pf:thm:consisttwolayer}

The proof relies on the following lemma that controls the regularization bias by imposing smoothness assumption on the truth.

\begin{lemma}[Hilbert scale and regularization bias] \label{lm:regbias}
Assume the operator $A$ in \eqref{eq:sem} is injective and compact. Let $f^\alpha = \operatorname{argmin}_{f\in \cH} \tfrac{1}{2} \|Af - b\|^2_\mathcal{E} + \tfrac{\alpha}{2} \| f\|^2_{\mathcal{H}} $ for some $\alpha > 0$. If the solution $f$ to \eqref{eq:sem} lies in the regularity space $\Phi_\beta$ defined in \eqref{eq:hilberstscale} for some $\beta > 0$, then
\[ \| f - f^\alpha\|_\cH^2 = O(\alpha^{\min\{ \beta, 2\}}).\]
\end{lemma}
\begin{proof}
See Section 3.3 of \citep{carrasco2007linear}.
\end{proof}

Compactness of a conditional expectation operator is a mild condition; see Appendix \ref{app:compactofcondexpop} for a discussion.

We remark that four quantities are involved in this proof: the truth $f$ that uniquely solves $Af = b$, the Tikhonov regularized solution $f^\alpha$ defined in \eqref{eq:tikhnovov2}, the Tikhonov regularized solution approximated by the class of NNs (see Equation \eqref{eq:primewithnn}), denoted $f^\alpha_{\text{NN}}$, and the average of the iterates generated by \ref{ppoalgo}, $\widebar f_T$. Lemma \ref{lm:regbias} provides a bound on the gap between $f$ and $f^\alpha$; Theorem \ref{thm:globalconvergence} controls $f^\alpha_{\text{NN}} - \widebar{ f_T}$. Theorem \ref{thm:consisttwolayer} assumes that $f^\alpha_{\text{NN}} = f^\alpha$. See Section \ref{app:roadmap} for a graphical representation.

We start with the decomposition of $\|\widebar f_T - f \| _\cH ^2$
\[\|\widebar f_T - f \| _\cH ^2 \leq 2 \|\widebar f_T - f^{\alpha} \| _\cH ^2 + 2 \| f^\alpha - f \| _\cH ^2.\]
Here the first term on the RHS represents optimization error and the second term is regularization bias. Lemma \ref{lm:regbias} provides a bound on the second term. Now we bound the first term.

Recall the definition of Tikhonov regularized functional for a compact linear operator $A$
\[L(f) = L_\alpha(f) = \frac{1}{2} \|Ag -b \|_{\cE}^2 + \frac{\alpha}{2} \|f\|_{\cH}^2.\]
Denote by $f^\alpha$ the unique minimizer of $L$ over $\cH$. This is always well-defined for a compact linear operator $A$. We want to show the strong convexity of $L_\alpha$, i.e.,
\# \label{eq:strongcovexofTik}
\frac{\alpha}{2} \| \widebar f_T - f^{\alpha}\|_\cH^2 \leq L_\alpha(\widebar f_T) - L_\alpha (f^\alpha).
\#
If \eqref{eq:strongcovexofTik} is true, under the conditions of Theorem \ref{thm:globalconvergence} (2-layer NN case), we have with probability at least $1-\delta$ over the sampling process,
\# 
\E_\init [ \| \widebar f_T - f^{\alpha}\|_\cH^2 ] & \leq \frac{2}{\alpha} \E_\init [L_\alpha(\widebar f_T) - L_\alpha (f^\alpha)] 
\\
&= \frac{2}{\alpha}\cO\Big( a\eta B + \frac{B}{T\eta} + \frac{a B^{3/2}\log^{1/2}(1/\delta)}{T^{1/2}} + \frac{aB^{5/2}}{m^{1/4}}\Big). \label{05100454}
\#
Setting $\eta = (aT)^{-1/2}$ where $a = \max\{\alpha, 1 \}$, and combining results from Lemma \ref{lm:regbias} and \eqref{05100454} we complete the proof.

Now we show \eqref{eq:strongcovexofTik}. For all $x \in \cH$, $x+h \in \cH$,
\begin{align}
2 {L_\alpha}(x+h)
&=
\|A(x+h)-b\|_{\cE}^2+\alpha\|x+h\|_{\cH}^2
\\&=
\|Ax-b\|_{\cE}^2+\|Ah\|_{\cE}^2 + \langle Ax-b,Ah\rangle_{\cE}+\alpha\|x\|_{\cH}^2+\alpha\|h\|_{\cH}^2+2\alpha\langle x,h\rangle_{\cH}
\\&=
2{L_\alpha}(x)+2\alpha\langle x,h\rangle_{\cH} + 2\langle Ax-b,Ah\rangle_{\cE}+\|Ah\|_{\cE}^2+\alpha\|h\|_{\cH}^2
\\&=
2{L_\alpha}(x)+2\alpha\langle x,h\rangle_{\cH} + 2\langle A^*(Ax-b),h\rangle_{\cH}+\|Ah\|_{\cE}^2+\alpha\|h\|_{\cH}^2
\\&=
2{L_\alpha}(x)+2\langle\alpha x+A^*Ax - A^*b,h\rangle_{\cH}+\|Ah\|_{\cE}^2+\alpha\|h\|_{\cH}^2.
\end{align}
Moreover, the regularized solution $f^{\alpha}$ is given by the unique solution to the equation $\alpha f^{\alpha}+A^{*} A f^{\alpha}=A^{*} b$ and depends continuously on $b$ \citep{kress1989linear}. Setting $x = f^\alpha$, $h = f -f^\alpha$ and applying $\alpha f^\alpha + A^*Af^\alpha = A^* b$ complete the proof of \eqref{eq:strongcovexofTik}.

\section{Compactness of conditional expectation operators} \label{app:compactofcondexpop}
Let $X = [X_1^\top, X_2^\top]\tp$ be a random vector with distribution $F_X$ and let $F_{X_1}, F_{X_2}$ be the marginal distributions of $X$ and $Y$, respectively. Assume there is no common elements in $X_1$ and $X_2$. Define Hilbert spaces $\cH = L^2(X_1)$ and $\cE = L^2(X_2)$. Let $A$ be the conditional expectation operator:
\$
 A:& \cH \to \cE
\\
& f(\cdot) \to \E[f(X_1) \mid X_2 = \cdot \,]\,.
\$
 
If there is no common elements in $X_1$ and $X_2$, compactness of an conditional expectation operator is in fact a mild condition \citep{carrasco2007linear}. If p.d.f.s of $X,X_1$ and $X_2$ exist, denoted $f_X, f_{X_1}$ and $f_{X_2}$, then $A$ can be represented as an integral operator with kernel 
\$
k(x_1, x_2)=\frac{f_{X_1, X_2}(x_1, x_2)}{f_{X_1}(x_1) f_{X_2}(x_2)},
\$
 and $(Af)(x_2) = \int k(x_1,x_2) f(x_1) f_{X_1}(x_1)dx_2$. In this case, a sufficient condition for compactness of $A$ is 
\[
\iint\left[\frac{f_{{X_1}, {X_2}}({x_1}, {x_2})}{f_{{X_1}}({x_1}) f_{{X_2}}({x_2})}\right]^{2} f_{{X_1}}({x_1}) f_{{X_2}}({x_2}) {d} {x_1} {d} {x_2}<\infty.
\]

 We now discuss well-posedness of \eqref{eq:sem}. The operator equation \eqref{eq:sem} is called well-posed (in Hadamard's sense) if (i) \textit{(existence)} a solution $f$ exists, (ii) \textit{(uniqueness)} the solution $f$ is unique, and (iii) \textit{(stability)} the solution $f$ is continuous as a function of $b$. More precisely, if $A: \mathcal{H} \to \mathcal{E}$ is bijective and the inverse operator $A^{-1}$ is continuous, then equation \eqref{eq:sem} is well-posed \citep{kress1989linear}. Injectivity is usually a property of the data distribution and is tantamount to assuming identifiability of the structural function

\section{A comment on Dual IV} \label{app:dualivcomment}
In this section, we review the work of Dual IV \citep{muandet2019dual} and point out the differences between their work and ours. Dual IV considers nonparametric IV estimation using min-max game formulation and bears similarities with this work. However, we remark that our framework \eqref{eq:sem} includes a wide range of models, including IV regression, and that the use of NNs and detailed analysis on the convergence of the training algorithm also distinguish our work from Dual IV. The goal of this section is to show the resulting min-max problem for IV regression in this paper has a natural connection with GMM.

Recall that IV regression considers the following conditional moment equation
\# 
\E[Y - g(X) \mid Z] = 0. \tag{\ref{eq:iv}, revisited}
\# 
Let $\mathcal{G}$ be an arbitrary class of continuous functions which we assume contains the truth that fulfills the integral equation. Dual IV proposes to solve
\# \label{eq:dualiv}
\min _{g \in \mathcal{G}} R(g):=\mathbb{E}_{Y Z}\left[  \big(Y - \mathbb{E}[g(X) \mid Z] \big) ^ 2 \right],
\#
while this paper solves
\[\min _{g \in \mathcal{G}} L(g)=\|Af - b\|_\cE^2 = \mathbb{E}_{ Z}\left[  \big( \E[Y\mid Z] - \mathbb{E}[g(X) \mid Z] \big) ^ 2 \right],\]
an unregularized version of Example \ref{ex:iv}. The operator $A$ and $b\in \cE$ are defined in Example \ref{ex:iv}.

To introduce the maximizer, Dual IV \citep{muandet2019dual} resorts to the interchangeability principle.
\begin{lemma}[Interchangeable principle] \label{lm:exchange}
Let $(\Omega, \mathcal{F}, \P)$ be a probability space, $f: \R^n \times \Omega \rightarrow \mathbb{R} \cup\{+\infty\}$, and $\mathcal{L}_2 = \mathcal{L}_2(\Omega, \cF, \P)$ be the class of square integrable functions. Let $\mathcal{X}$ be the set of mappings $\chi: \Omega \rightarrow \R^n$ such that $f_{\chi} \in \mathcal{L}_2$, where $f_{\chi}(\cdot):=f(\chi(\cdot), \cdot)$. Assume $F(\omega):=\sup _{x \in X} f(x, \omega) \in \mathcal{L}_2$ and that $f$ is upper semi-continuous\footnote{Random upper semi-continuous, to be precise.}. Then the following holds.
\[\mathbb{E}\Big[\sup _{x \in X} f(x, \omega)\Big]=\sup _{\chi \in \mathcal{X}} \mathbb{E}[f(\chi(\omega), \omega)].\]
\end{lemma}
\begin{proof}
See Proposition 2.1 in \citep{shapiro2017interchangeability}. See also Proposition 1 in \citep{dai17a} for a proof for the case where $f:\R \times \Omega \rightarrow \R$.
\end{proof}

With the interchangeability principle, \eqref{eq:dualiv} can be rewritten as 
\[\min _{g \in \mathcal{G}} \max _{u \in \mathcal{U}} \Psi(g, u):=\mathbb{E}_{X Y Z}[(g(X)-Y) u(Y, Z)] -\tfrac{1}{2} \mathbb{E}_{Y Z}\left[u(Y, Z)^{2}\right].\]
By comparison, an unregularized version of the min-max problem derived in this paper \eqref{eq:ivminmaxwithreg} is
\# \label{eq:ivminmaxwithout}
\min_{g\in L^2(X)} \max_{u \in L^2(Z)} \E_{XYZ}[ ( g(X) - Y ) \cdot u(Z) - \tfrac12u^2(Z) ]. 
\#

The absence of the variable $Y$ in the maximizer $u$ in \eqref{eq:ivminmaxwithout} facilitates a natural connection between \eqref{eq:ivminmaxwithout} and GMM. 

To achieve such interpretation, we first introduce a GMM estimator for \eqref{eq:iv}. The conditional moment restriction \eqref{eq:iv} implies that for a set of functions $f_{1}, f_{2}, \ldots, f_{m}$ of $Z$, we have $\mathbb{E}\left[ (Y-g(X)) f_{j}(Z)\right]=0$. Define by $\psi(f, g):=\mathbb{E}_{X Y Z}[(Y-g(X)) f(Z)]$ the moment violation function, and the GMM estimator
\[g_{\mathrm{GMM}} \in \arg \min _{g \in \mathcal{G}} \frac{1}{2} \sum_{j=1}^{m} \psi\left(f_{j}, g\right)^{2}.\] Collect the moment violations by a vector ${\psi_v}(g):=\left(\psi\left(f_{1}, g\right), \ldots, \psi\left(f_{m}, g\right)\right)^{\top} \in \mathbb{R}^{m}$. To achieve efficiency the moments are usually weighted. Let $\Lambda$ be a $m$ by $m$ symmetric matrix. We define the quadratic norm $\| \phi \|_\Lambda^2 = \phi \tp \Lambda \phi$ given a vector $\phi$.

Now we are ready to state the connection between GMM and \eqref{eq:ivminmaxwithout}. Define the
space of maximizer $\mathcal{U}=\operatorname{span}\left\{f_{1}, \ldots, f_{m}\right\}$. We focus on the inner maximization of \eqref{eq:ivminmaxwithout}. Define
 \[J(g) \coloneqq \max_{u \in \cU }  \E_{XYZ}[ ( g(X) - Y ) \cdot u(Z) - \tfrac12u^2(Z) ].\]
Note that maximizer is now constrained in $\cU$. Mimicking Theorem 5 in \citep{muandet2019dual}, we can show $J(g)$ is in fact a weighted sum of the moment violations $\{\psi(f_j,g)\}$. 
\begin{lemma} \label{lm:ivandgmm}
Let $f_{1}, f_{2}, \ldots, f_{m}$ be a set of real-valued functions of $Z$. Define the weight matrix $\Lambda:=\mathbb{E}_{Z}[\mathbf{f}(Z)  \mathbf{f}(Z) \tp ]$ where $\mathbf{f}:=\left(f_{1}( Z), \ldots, f_{m}( Z)\right)^{\top}$. Then $J(g) =\tfrac12 \| \psi_v (g)\|_{\Lambda^{-1}}^2$.
\end{lemma}
\begin{proof}
The proof is identical to Appendix C of \citep{muandet2019dual} except for replacing $f(Y,Z)$ with $f(Z)$. The proof relies on simple algebra manipulation and is presented for completeness.
For any $u \in \mathcal{U}, u=\sum_{j=1}^{m} \alpha_{j} f_{j}$ for some $\boldsymbol{\alpha} = \left(\alpha_{1}, \ldots, \alpha_{m}\right)^{\top} \in \mathbb{R}^{m}$.
\$
J(g)& =\max _{\alpha \in \mathbb{R}^{m}} \mathbb{E}_{X Y Z}\left[(g(X)-Y)\left(\sum_{j=1}^{m} \alpha_{j} f_{j}( Z)\right)\right] -\frac{1}{2} \mathbb{E}_{ Z}\left[\left(\sum_{j=1}^{m} \alpha_{j} f_{j}( Z)\right)^{2}\right]
\\
& = \max _{\boldsymbol{\alpha} \in \mathbb{R}^{m}} \sum_{i=1}^{m} \alpha_{j} \mathbb{E}_{X Y Z}\left[(g(X)-Y) f_{j}( Z)\right] -\frac{1}{2} \mathbb{E}_{ Z}\left[\left(\sum_{j=1}^{m} \alpha_{j} f_{j}( Z)\right)^{2}\right]
\\
& =\max _{\boldsymbol{\alpha} \in \mathbb{R}^{m}} \boldsymbol{\alpha}^{\top} {\psi_v}-\frac{1}{2} \boldsymbol{\alpha}^{\top} {\Lambda} \boldsymbol{\alpha}
\\
& = \frac{1}{2} \psi_v^{\top} \Lambda^{-1} \psi_v.
\$
\end{proof}

Lemma \ref{lm:ivandgmm} shows that if the maximizer is constrained to be in the span of a set of pre-defined test functions $\{f_j\}$, the minimization in \eqref{eq:ivminmaxwithout} in fact produces a weighed GMM estimator. In contrast, the GMM interpretation provided in Section 3.5 of \citep{muandet2019dual} requires the definition of a so-called augmented IV $W := ( Y,Z)$. It is unnatural to view the response variable $Y$ as a component of the IV.


 \newpage
\section{A roadmap to the proof of Theorem
\ref{thm:consisttwolayer}} \label{app:roadmap}
In Figure \ref{fig:roadmap} we can see throughout the discussion we make a couple of simplifying assumptions (e.g., Assumption \ref{as:containtruefunction} assumes the conditional expectation operator is close in $\cF_{\text{NN}}$, and Assumption \ref{as:zeroapprox} assumes the primal problems \eqref{eq:primewithnn} and \eqref{eq:tikhnovov2} give the same solution). These assumptions are justified by the representation power of NNs. One could instead explicitly incorporate approximation error in the bounds.
\tikzstyle{mybox} = [    draw=gray]

\begin{figure}[!h]
    \centering
    
  \begin{tikzpicture}

\node[mybox]  at (0,0) (box1) {%
    \begin{minipage}{0.40\textwidth}
    \$
      Af=b \tag{\ref{eq:sem}}
      \\
      \Rightarrow f
      \$
    \end{minipage}
};

\node[mybox]  at (0,-3) (box2){%
    \begin{minipage}{0.35\textwidth}
     \$ & \argmin_{f\in \cH} \| Af-b\|_\cE + \tfrac{\alpha}{2}
        \|f\|_\cH^2 \tag{\ref{eq:tikhnovov2}}
       \\ 
        & \Rightarrow f^\alpha
    \$
    \end{minipage}
};

\node[mybox] at (8,-3) (box3) {%
    \begin{minipage}{0.35\textwidth}
     \$ & \argmin_{f\in \cF_{\text{NN}}} \| Af-b\|_\cE + \tfrac{\alpha}{2}
        \|f\|_\cH^2  \tag{\ref{eq:primewithnn}}
       \\ 
        & \Rightarrow f^\alpha_{\text{NN}}
    \$
    \end{minipage}
};

\node[mybox]  at (8,-6) (box4) {%
    \begin{minipage}{0.40\textwidth}
    \$
     & \min_{f\in \cF_{\text{NN}}} \max_{u\in \cF_{\text{NN}}}
 \phi(f,u)  \tag{\ref{eq:minmaxwithnn}}
       \\ 
        & \Rightarrow \text{A saddle point}
    \$
    \end{minipage}
};

\node[mybox] (box5) at (8,-9){%
    \begin{minipage}{0.40\textwidth}
     \$
     \text{\ref{ppoalgo}}
     \\
     \Rightarrow \widebar f_T
    \$
    \end{minipage}
};

\draw [<-] (box1) -- node [right]{Assumption \ref{as:reg}, Lemma \ref{lm:regbias}} (box2);

\draw [<-] (box2) -- node [above]{Assumption \ref{as:zeroapprox}} (box3);

\draw [<-] (box3) -- node [right]{Assumption \ref{as:containtruefunction}} (box4);

\draw [<-] (box4) -- node [left]{Convergence of SGD, NN linearization} (box5);
  \end{tikzpicture}

    \caption{Relation between the quantities of interest. Texts above/near the arrows summarize the key elements of connecting different problems.}
    \label{fig:roadmap}
\end{figure}
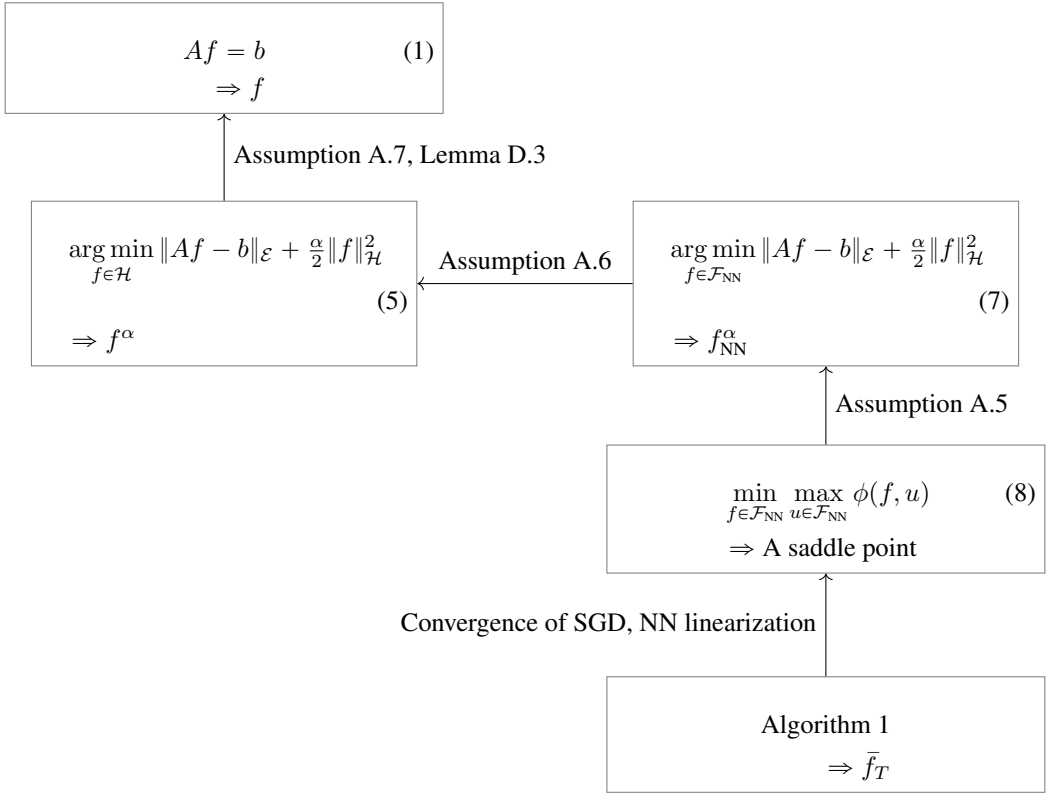

\end{appendix}
\end{document}